%% file: arxiv-clustering.tex
\documentclass[12pt]{article}
\usepackage[margin=1in]{geometry}

\usepackage{makecell,colortbl}
\usepackage{microtype}
\usepackage{graphicx}
\usepackage{subfigure,arydshln,makecell,overpic}
\usepackage{booktabs} 

\usepackage{times,rotating}
\usepackage{url}
\usepackage{booktabs, multicol, multirow}
\usepackage{morenotations2}
\usepackage{caption}

\usepackage{hyperref}
\usepackage[capitalize,noabbrev]{cleveref}
\usepackage{natbib}

\usepackage[capitalize]{cleveref}

\title{\papertitle}
\author{
  Ehsan Amid \qquad Richard Nock \qquad Manfred Warmuth \\
 Google Research\\
{\normalsize $\{$eamid,richardnock,manfred$\}$@google.com} \\
}
\begin{document}

\date{}

\maketitle

\begin{abstract}
\input{content/abstract}
\end{abstract}

\input{content/introduction}
\input{content/related}

\input{content/t-exp-fams}
\input{content/clustering}
\input{content/experiments}
\input{content/discussion}
\input{content/conclusion}
\input{content/acknowledgments}

\bibliography{bibgen,content/refs}
\bibliographystyle{apalike}

\newpage
\appendix
\onecolumn
\renewcommand\thesection{\Roman{section}}
\renewcommand\thesubsection{\thesection.\arabic{subsection}}
\renewcommand\thesubsubsection{\thesection.\thesubsection.\arabic{subsubsection}}

\renewcommand*{\thetheorem}{\Alph{theorem}}
\renewcommand*{\thelemma}{\Alph{lemma}}
\renewcommand*{\thecorollary}{\Alph{corollary}}

\renewcommand{\thetable}{A\arabic{table}}

\begin{center}
\Huge{Appendix}
\end{center}

To
differentiate with the numberings in the main file, the numbering of
Theorems, etc. is letter-based (A, B, ...).

\section*{Table of contents}

\noindent \textbf{Supplementary material on proofs} \hrulefill\\
\noindent $\hookrightarrow$ Cheatsheet for $t$-functions, $t$-algebra and related functions\hrulefill Pg \pageref{sec-cheatsheet}\\
\noindent $\hookrightarrow$ Proof of Theorem \ref{thTEXPM}\hrulefill Pg \pageref{sec-proof-thTEXPM}\\
\noindent $\hookrightarrow$ Proof of Lemma \ref{lemMf}\hrulefill Pg \pageref{sec-proof-lemMf}\\
\noindent $\hookrightarrow$ Proof of Theorem \ref{thm-ITIG}\hrulefill Pg \pageref{sec-proof-thm-ITIG}\\
\noindent $\hookrightarrow$ Proof of Lemma \ref{lem-rob-right}\hrulefill Pg \pageref{sec-proof-lem-rob-right}\\
\noindent $\hookrightarrow$ Proof of Lemma \ref{lem-left-popmin}\hrulefill Pg \pageref{sec-proof-lem-left-popmin}\\
\noindent $\hookrightarrow$ Proof of Lemma \ref{lem-rob-left}\hrulefill Pg \pageref{sec-proof-lem-rob-left}\\
\noindent $\hookrightarrow$ Proof of Theorem \ref{th-conti}\hrulefill Pg \pageref{sec-proof-th-conti}\\

\noindent \textbf{Supplementary material on experiments} \hrulefill Pg \pageref{sec-voronoi}

\newpage

\input{content/appendix-proofs}
\input{content/appendix-experiments}







\end{document}

%% file: content/abstract.tex
The link with exponential families has allowed $k$-means clustering to be generalized to a wide variety of data generating distributions in exponential families and clustering distortions among Bregman divergences. Getting the framework to work above exponential families is important to lift roadblocks like the lack of robustness of some population minimizers carved in their axiomatization. Current generalisations of exponential families like $q$-exponential families or even deformed exponential families fail at achieving the goal. In this paper, we provide a new attempt at getting the complete framework, grounded in a new generalisation of exponential families that we introduce, \textit{tempered} exponential \textit{measures} (\acrotem). \acrotem s keep the maximum entropy axiomatization framework of $q$-exponential families, but instead of normalizing the measure, normalize a dual called a co-distribution. Numerous interesting properties arise for clustering such as improved and controllable robustness for population minimizers, that keep a simple analytic form.


%% file: content/introduction.tex
\section{Introduction}\label{sec-int}

Decades after its introduction \citep{lLS,sSL}, $k$-means remains a hugely popular algorithm \citep{fML,htfTE} with a very active research agenda \citep{pcdxUC,vcxBP}. Generally speaking, clustering is a loosely formulated problem, in particular in terms of function to optimise if we compare it to supervised learning \citep{vwgCS} and even objects to be clustered in terms of their complexity range \citep{bLP}. $k$-means has a comparative advantage over other techniques from these two standpoints: the objects clustered are equivalently the expectation parameters of Gaussians with identity covariance \cite[pp 17]{ngSE}, so there is a sound statistical interpretation to the objects being clustered (or the parameters learned), and the eventual generative process of the training data \citep{vcxBP}. Also, the loss optimized, a Bregman divergence known as squared Mahalanobis distance, stems from the KL divergence between two such Gaussians, thus having solid information theoretic grounds. This very elegant property can be extended ``above'' the Gaussian distribution to any \textit{exponential family} \citep{bmdgCWj}, generalizing the clustering losses used to general Bregman divergences, and it has a practical impact on improved design in specific application areas \citep{fbdNN}.

In fact, this property can be extended further \textit{above} exponential families, towards $q$- and deformed- exponential families using escort distributions \cite[Theorem 3]{ncmqwFG}, \citep{aomGO,vcOP}, \textit{but} there is no more ``novelty'' on the parameters' side as Bregman divergences are kept as distortions between parameters.

Getting such novelty would be crucial for clustering: the cluster centers, also called \textit{population minimizers} that elicit the most general clustering algorithms belong to a small set from the analytic standpoint, with one, the arithmetic average, being ubiquitous for all Bregman divergences \cite[Proposition 1]{bmdgCWj}. This is an issue for clustering in terms of robustness to outliers \cite[Section 11.1.6]{aIG}. For example, the arithmetic average lacks robustness: adding a single point that progressively drifts away will drag a cluster center arbitrarily far away from its initial value, bringing considerable instability to clustering. A solution to this problem cannot easily arise within exponential families, nor $q$-exponential nor deformed exponential families because the arithmetic average as maximum likelihood estimator is carved in their axiomatisation \cite[pp 137]{bnIA}. Adding robustness is not necessarily an issue by going ``above'' Bregman divergences \citep{nnaOCD,vlanTB}, \textit{but} either the connection with distributions is lost or substantially departs from exponential families. This task is not trivial since it has to go through generalizing all key objects at play, thus including (i) the distributions (\textit{i.e.}, generalizing exponential families), (ii) information-theoretic distortions between distributions (KL divergence), (iii) parameter-based clustering distortions (Bregman divergences), (iv) an actionable identity between the distortions in (ii) and (iii), and of course (v) population minimizers ($f$-means).

We know of no approach that gets above exponential families and covers (i) through (v) while conveniently expanding the realm of distortions beyond Bregman divergences. 

\textbf{Our paper is a proposal that achieves this goal}. While our contributions thus span all steps from (i) to (v), the benefit for downstream clustering is simple: it provides improved robustness for population minimizers. Technically speaking, our key thread is close to Tsallis' nonextensive statistics framework \citep{tIT}, inclusive of the specific arithmetic developed in its context \citep{nlwGA}, \textit{but} with an early tweak: we do not normalize the solution of the maximum entropy but a dual that we call a \textit{co-(tempered exponential) distribution} (\acroted). The unnormalized solution is called \textit{tempered exponential measure} (\acrotem). This is a big difference with work that followed the Amari-Naudts-Tsallis $q$-exponential families, deformed exponential families, and their escort distributions, which are all normalized \citep{aIG,nGT,tIT}. \acrotem/\acroted~depend on a parameter $t$ and as $t\rightarrow 1$, both converge to the same exponential family. Getting unnormalized measures is the trick that brings improved robustness for clustering, by creating an \textit{unbalanced} clustering problem whose parameter distortions, generalizing Bregman divergences, belong to a broad subset known as conformal Bregman divergences \citep{nnaOCD}. 

Our results in (i) to (v) have wider interest than clustering, so as for additional results we use, such as simple and elegant closed forms for key functions including the cumulant (Theorem \ref{thTEXPM}, unlike, \textit{e.g.}, $q$-exponential families) and the total mass of the \acrotem~(Lemma \ref{lemMf}), etc. .

To ease reading, all proofs and additional experiments are given in an Appendix, denoted for short as \supplement.

%% file: content/related.tex
\section{Problem and related work}\label{sec-rel}

For space constraints, we shall reduce technicalities and jargon related to exponential families to their minimum. We refer to textbooks in mathematical statistics \cite[Chapter 8]{bnIA} or information geometry \cite[Section 4.2]{anMO} for extensive coverage. An exponential family can be obtained by maximizing Shannon's entropy subject to normalization and conditions on the arithmetic average being the maximum likelihood estimator \citep{bnIA}; its density has the general form
\begin{eqnarray}
p_{\ve{\theta}}(\ve{x}) & \propto & \exp(\ve{\theta}^\top \ve{\phi}(\ve{x}) - G(\ve{\theta})), \label{defExpFam}
\end{eqnarray}
where $\ve{\phi}$ is the sufficient statistic, $\ve{\theta}$ is the natural parameter and $G$, the cumulant or partition function, ensures normalisation (the ``$\propto$'' symbol simplifies the carrier or base measure). The natural parameter holds the information about the ``individual'' distribution inside its family, encoded in $G$. The connection between exponential families and clustering \textit{\`a-la}-$k$-means is simple to state and enlightening on what such clustering achieves. Given any two distributions $P_i, P_j$ with densities $p_i, p_j$, a popular information-theoretic distortion measure for their comparison is an $f$-divergence \citep{asAG,cEI}, with one especially important in our context, the reverse KL-divergence:
\begin{eqnarray}
F(P_i\|P_j) & \defeq & \int f(\mathrm{d} p_i / \mathrm{d} p_j) \mathrm{d}p_j,\,\,\, f \defeq -\log.\label{defFDIV}
  \end{eqnarray}
  Suppose then we have a set of distributions $\{P_i\}_{i=1}^m$ and wish to find a set of distributions $\{Q_j\}_{j=1}^k$, $k$ being user-fixed, minimizing the following loss function:
  \begin{eqnarray}
F(\{P_i\}_{i=1}^m,\{Q_j\}_{j=1}^k) & \defeq & \expect_i[\min_{j} F(P_i\|Q_j)].\label{defCLUD}
  \end{eqnarray}
  Without any further assumption, this well-founded formulation of the clustering problem falls at two hurdles: (i) the potential intractability of the integrals to compute \eqref{defFDIV} and (ii) the formulation and/or computation of the so-called population minimizers $Q_.$ in \eqref{defCLUD}. A simple assumption solves both problems simultaneously: if all distributions are assumed to belong to the \textit{same} exponential family, then
\begin{eqnarray}
F(P_i\|Q_j) & = & D_{G}(\bm{\theta}_i \| \bm{\vartheta}_j),\label{defCLUNP}
\end{eqnarray}
the Bregman divergence between the natural parameters and with generator $G$ (assumed to be strictly convex differentiable), is
\begin{eqnarray}
  D_{G}(\bm{\theta}_i\|\bm{\vartheta}_j) & \defeq & G(\bm{\theta}_i) - G(\bm{\vartheta}_j) -(\bm{\theta}_i - \bm{\vartheta}_j)^\top\nabla G(\bm{\vartheta}_j).\label{defBD}
\end{eqnarray}
The original $k$-means clustering is obtained for $D_G$ being squared Mahalanobis distance, which corresponds to distributions being Gaussians with identity covariance. For any Bregman divergence, the \textit{right} population minimizer in \eqref{defCLUD} is \textit{always} the average \citep{bmdgCWj}. This allows generalizing the $k$-means algorithm to all Bregman divergences by repeatedly allocating points to their closest center Bregman-wise and updating cluster centers with their cluster's average. A Bregman divergence being asymmetric in general, one can choose to flip arguments in \eqref{defCLUNP}: the \textit{left} population minimizer is then an $f$-mean of the form $\nabla G^{-1} \expect \nabla G(.)$.

To summarize, $k$-means clustering operates in disguise on \textit{parameters} of distributions using distortions that can be understood from both the information geometric \eqref{defCLUNP} and information-theoretic \eqref{defFDIV} standpoints. Such distributions can naturally be related to a generative process for the observed data and the whole algorithm can also be understood from a Bayesian standpoint \citep{nBM} where priors and posteriors are modeled with the initial ``guess'' of an exponential family. All key steps to get the complete characterisation are steps (i) to (v) sketched in the introduction.

\textit{In the context of clustering}, a relevant question is to get this scheme to work beyond its restriction of the ``same exponential family'' assumption. Alleviating the ``same'' is not straightforward: removing this assumption decomposes the KL divergence in a sum of two Bregman divergences, one between the cumulants \cite[Theorem 24]{ncmqwFG}. More important is, in fact getting above the ``exponential family'' assumption because the population minimizers -- in particular the average -- can suffer from a lack of robustness, but this lack of robustness is, up to some extent, carved in the axiomatic definition of exponential families \cite[pp 137]{bnIA}, \cite[Section 2.8.1]{aIG} and Bregman divergences \citep{bgwOT}.

Natural candidates to go above exponential families are $q$-exponential families and deformed exponential families \citep{aIG,aomGO,nGT}. $q$-exponential families essentially replace the $\exp$ in \eqref{defExpFam} by a generalization, the $q$-exponential:
\begin{eqnarray}
\exp_q (z) & \defeq & \left[1+(1-q) z\right]^{1/(1-q)}_+,\label{defExpQ}
\end{eqnarray}
with $[z]_+ \defeq \max\{0,z\}$ and $q\geq 0$ ($q>0$) guarantees the convexity (strict) of the function. Deformed exponential families go further in the generalization by replacing the $q$-exponential by a $\chi$-exponential for some $\chi$ positive non-decreasing:
\begin{eqnarray}
\exp_q (z) & \defeq & \left(\int_{1}^{z}\frac{1}{\chi(t)} \mathrm{d}t\right)^{-1},\label{defExpChi}
\end{eqnarray}
the $q$-exponential being derived for $\chi(z) \defeq z^q$. A dual distribution can be derived in all cases, called an \textit{escort}, whose density has the general form $\tilde{p}_{\ve{\theta}}(\ve{x}) \,\propto\, \chi(p_{\ve{\theta}}(\ve{x}))$.

It turns out neither $q$-exponential nor deformed exponential families can fulfill our needs, because their equivalent of \eqref{defCLUNP} \textit{still} involves a Bregman divergence on the parameter side, see for example \citet[Theorem 3]{ncmqwFG}. To get robustness, one previous work departs from both Bregman divergences and exponential families, \citet{lvanSR}: in this case, the Bregman divergence, which computes the difference between a convex function and a tangent plane, is replaced by the distance to the projection on a tangent plane, called a total Bregman divergence. A link is established with distributions but these are substantially different from exponential families as their natural parameters belong to a submanifold defining a curved family of distributions. Our objective is rather to go above exponential families with a sufficient broadening of the Bregman divergence part. Ideally, the divergence part would pave way for new properties such as improved robustness for clustering and the distribution part, beyond generalizing exponential families, would include guarantees of ``proximity'' to exponential families as new properties on the parameters' side appear. This is important given the ubiquitous nature of exponential families as a tool in ML. 

Finally, we also note a recent breakthrough tied to exponential families which, instead of an information-theoretic - information-geometric link as in \eqref{defCLUNP} establishes a regularised optimal transport - information geometric link \citep{jmpcEO}, but only for Gaussian measures (not necessarily normalized).

%% file: content/t-exp-fams.tex
\newcommand{\wrapRN}[1]{}

\section{Tempered exponential measures and their co-densities}\label{sec-exp}

We make extensive use of the $q$-exponential function defined in \eqref{defExpQ}; in our context, parameter $q$ is renamed $t$ to make a clear distinction of the notations we use. We define the inverse of the $t$-exponential~\citep{nGT}:
\begin{eqnarray}
\log_t(z) \defeq \frac{1}{1-t} \left(z^{1-t}-1\right) \quad \quad \left(\lim_{t\rightarrow 1} \log_t = \log\right)
\end{eqnarray}
We introduce notions of duality using $t$.
\begin{definition}
  The dual $t^*$ of $t$ is $t^* \defeq 1/(2-t)$; the dual $(\exp_t)^*$ of $\exp_t$ is the perspective transform:
  \begin{eqnarray}
\left(\exp_t\right)^* (z) & \defeq & t^* \exp_{t^*} \left(\frac{z}{t^*}\right).
  \end{eqnarray}
  Last, we define in the same way the dual $(\log_t)^*$ of $\log_t$.
  \end{definition}
  We remark that if $t\in [0,1]$ then $t^* \in [1/2, 1]$. As already outlined in the introduction, we shall make use of unnormalized measures -- and by extension, unnormalized densities -- when dealing with such objects, a \textit{tilda} shall indicate it is \textit{not necessarily normalized}. The following gives the first example, where $\ve{\phi}: \mathcal{X} \rightarrow \mathbb{R}^d$ denotes a sufficient statistics and $\ve{\hbar}$ an expectation parameter (boldfaces are used for vector notations). \wrapRN{\eamid{Do we need to mention the base measure $\mathrm{d} \xi$?}}
  \begin{eqnarray}
    \tilde{\mathcal{P}}_{t|\ve{\hbar}} \defeq \left\{\tilde{p} \left|
                                                  \begin{array}{l}
\expect_{\tilde{P}}[\ve{\phi}] \defeq \int \ve{\phi}(\ve{x})\, \tilde{p}(\ve{x})\, \mathrm{d} \xi = \ve{\hbar},  \\
                                                    \int \tilde{p}(\ve{x})^{\Red{$1/t^*$}}\, \mathrm{d} \xi  = 1, \\
                                                   \tilde{p}(\ve{x}) \geq 0, \forall \ve{x}\in\mathcal{X}.
                                                    \end{array}
                                                    \right.\right\}\label{defPTILDE}
    \end{eqnarray}
    denotes a set of unnormalized densities.\footnote{We omitted the domination condition of $\tilde{p}$ wrt $\xi$ for simplicity; importantly, the \textit{expectation} $\expect$ also uses the \textit{unnormalized} measure $\tilde{P}$. Checking the argument of an expectation allows to infer whether the inner density is normalized.} Following the classical approach, we elicit the element(s) of $\tilde{\mathcal{P}}_{t|\ve{\hbar}}$ whose maximizing a generalised notion of the Tsallis entropy (Capital $\tilde{P}$ denotes the measure of density $\tilde{p}$ wrt $\xi$):
    \begin{eqnarray}
H_t(\tilde{P}) &\defeq &  - \int \psi_t(\tilde{p}(\ve{x}))\, \mathrm{d} \xi, \label{eq:tsallis}\\
    \psi_t(z) & \defeq & z\log_t z - \log_{t-1} z. \label{eq:tsallis-entropy}
    \end{eqnarray}
    With Tsallis entropy\footnote{\citet{bitemp} introduce this slightly different notion of entropy which recovers the Itakura-Saito convex generator $\psi_2(z) = z - \log z - 1$ at $t \rightarrow 2$. Following the standard definition of the Tsallis entropy and discarding the last term in $\psi_t$ does not affect our construction. Also note that we handle the constant term in \eqref{eq:tsallis} by subtracting from the integrand and adding back outside of the integral. Thus, we assume that the constant $t^*$ inside the second term is unaffected by the integral.} and replacing in \eqref{defPTILDE} $\Red{$1/t^*$}$ by constant {\color{blue}{1}}, $\tilde{\mathcal{P}}_{t|\ve{\hbar}}$ would cover \textit{probability} density functions related to $t$=$q$-exponential families \citep{nEE} (in fact, their escorts). The change {\color{blue}{1}}$\rightarrow\Red{$1/t^*$}$ may look cosmetic in the definition but has dramatic consequences in the whole chain of results that leads from $\tilde{\mathcal{P}}_{t|\ve{\hbar}}$ to clustering. The first major difference is that $q$-exponential families do not admit a closed form expression for the cumulant $G$ in \eqref{defExpFam} \cite[p 12]{nEE}. Our solution \textit{does} and it is an elegant generalisation of that for exponential families, as we now prove. The theorem makes use of a generalization of the substraction, $\ominus_t$, in the \textit{$t$-arithmetic} introduced in \cite{nlwGA}:
    \begin{eqnarray}
z \ominus_t x & \defeq & \frac{z - x}{1+(1-t) x}.
      \end{eqnarray}
\begin{theorem}\label{thTEXPM}
For any $t\in [0,1]$ and $\ve{\hbar} \in \mathbb{R}^d$, the solution $\arg\max_{\tilde{\mathcal{P}}_{t|\ve{\hbar}}} H_t$ has the non-normalized density
\wrapRN{\mw{Why $\propto$ when we have $=$?}\noteRN{because it simplifies the base measure (otherwise, the proof would have to introduce it)}}
\begin{equation}
\label{eq:exp_t_density_form}
\tilde{p}_{t|\ve{\theta}} (\ve{x}) = \frac{\exp_t(\ve{\theta}^\top \ve{\phi}(\ve{x}))}{\exp_t(G_t(\ve{\theta}))} = \exp_t(\ve{\theta}^\top \ve{\phi}(\ve{x}) \ominus_t G_t(\ve{\theta})),\hspace{-0.2cm}
\end{equation}
where
\begin{eqnarray}
  G_t(\ve{\theta}) & = & (\log_t)^* \int (\exp_t)^* (\ve{\theta}^\top \ve{\phi}(\ve{x}))\mathrm{d}\xi\label{eqGT}
\end{eqnarray}
is the (convex) cumulant ensuring the normalization of the \textbf{dual} $\tilde{p}^{1/t^*}$; assuming $G_t$ differentiable, the correspondence $\ve{\theta} = \nabla G_t^{-1}(\ve{\hbar})$ also holds and $\ve{\theta}$ is called a natural parameter.
\end{theorem}
(Proof in \supplement, Section \ref{sec-proof-thTEXPM}) Hereafter, we assume $t\in [0,1]$, which is technically convenient in our context, but note that Theorem \ref{thTEXPM} operates on a wider range of $t$ values modulo eventual local tweaks: for example, for $t=2$, we have to discard the $-t^*$ term in our entropy \eqref{eq:tsallis-entropy}. We introduce the nomenclature of \textit{tempered exponential measures} (\acrotem) whose (non-normalized) densities are given by \eqref{eq:exp_t_density_form} and their \textit{co-densities} (\acroted) which are the (normalized) ``duals'' defined by $(\tilde{p}_{t|\ve{\theta}})^{1/t^*}$. Importantly, we note that since $\lim_{t\rightarrow 1} (\log_t)^* = \log$ and $\lim_{t\rightarrow 1} (\exp_t)^* = \exp$, $G_t$ in \eqref{eq:exp_t_density_form} is indeed the generalisation of the well-known expression for exponential families, $G$ in \eqref{defExpFam}. Such an expression is not known for $q$-exponential families. To properly define \acrotem s like exponential families \citep{ngSE}, we have to include an eventual carrier measure: we thus let $k(\ve{x})$ denote the carrier measure and let a general \acrotem~be defined from the unnormalized density:
\begin{equation}
\label{eq:exp_t_DF_2}
\tilde{p}_{t|\ve{\theta}} (\ve{x})\, =\, \exp_t(\ve{\theta}^\top \ve{\phi}(\ve{x}) \ominus_t G_t(\ve{\theta}) \oplus_t k(\ve{x}))
\end{equation}
where $\oplus_t$ was also introduced in \cite{nlwGA}:
    \begin{eqnarray}
z \oplus_t x & \defeq & z + t + (1-t)zx .
      \end{eqnarray}
      In addition to having a cumulant in nice form, \acrotem s have another key property: the total mass due to $\tilde{p}_{t|\ve{\theta}}$, $\mathrm{M}_t({\ve{\theta}})\defeq \int \tilde{p}_{t|\ve{\theta}}(\ve{x}) \mathrm{d} \xi$, is \textit{also} available in an elegant closed form. Hereafter, $G^\star_t$ denotes the convex conjugate of $G_t$.
\begin{lemma}\label{lemMf}
  $\mathrm{M}_t({\ve{\theta}}) = 1 + (1-t)(G_t(\ve{\theta}) - \ve{\theta}^\top\ve{\hbar})$. If $G_t$ is strictly convex differentiable,
  \begin{equation}
    \mathrm{M}_t({\ve{\theta}}) = 1 + (1-t) (-G^\star_t(\ve{\hbar}))
    \,\,(=\exp_t^{1-t} (-G^\star_t(\ve{\hbar}))).\label{eqMt}
  \end{equation}
\end{lemma}
(proof in \supplement, Section \ref{sec-proof-lemMf}) Naturally, we recover $\lim_{t \rightarrow 1} \mathrm{M}_t({\ve{\theta}}) = 1$ for exponential families. Since the total mass is positive by definition, we get two nontrivial bounds on the cumulant and its convex conjugate: $G_t(\ve{\theta}) \geq -1/(1-t) + \ve{\theta}^\top\ve{\hbar}$ and $G^\star_t(\ve{\hbar}) \leq 1/(1-t)$, both of which become vacuous when $t\rightarrow 1^-$.
\begin{table*}
  \centering
    \begin{tabular}{c|c|ccc|c}\hline\hline
      \acrotem & Support & $\ve{\lambda}$ & $\ve{\theta}$ & $\ve{\hbar}$ & $G^\star_t(\ve{\hbar})$ \\ \hline
    1D $t$-exponential & $\left[0, \frac{3-2t}{(1-t)\lambda}\right]$ & $\lambda$ &
                                                                    $\frac{-\lambda}{3-2t}$ & $t^* \left(\frac{3-2t}{\lambda}\right)^{2-t^*}$ & $-t^*\cdot\left(\log_{\frac{1}{2-t^*}}\left(\frac{\hbar}{t^*}\right) - 1\right)$\\
      1D $t$-Gaussian ($\mu = 0$) & $\left[- \frac{1}{\sqrt{1-t}}, \frac{1}{\sqrt{1-t}}\right]$ & $\sigma^2$ & $-\frac{t^*}{2\sigma^2}$ & $(c_{t^*} \sqrt{2})^{1-t^*} \sigma^{3-t^*}$ & $-\frac{t^*}{2} \cdot \left(\log_{t^{**}}(2c_{t^*}^2\hbar ) - 1\right)$\\ \hline\hline
    \end{tabular}
    
    \begin{tabular}{c|c|cc|cc}\hline\hline
      \acrotem & $G_t(\ve{\theta})$ & $B_{G_t}(\hat{\ve{\theta}} \| \ve{\theta})$ \\ \hline
      1D $t$-exponential & $-\log_{2-t}\left(\left(-\theta\right)^{\frac{1}{2-t}}\right)$ & $t^* \cdot  \left(\left(\frac{\hat{\theta}}{\theta}\right)^{2-t^*}-(2-t^*)\cdot \log_{t^*}\left(\frac{\hat{\theta}}{\theta}\right)-1\right)$\\ 
     1D $t$-Gaussian ($\mu = 0$) & $\left(\log_t\right)^* \left(\frac{c_{t^*}}{\sqrt{-\theta}}\right)$ & $\frac{t^*}{2} \cdot  \left(\left(\sqrt{\frac{\hat{\theta}}{\theta}}\right)^{3-t^*}-(3-t^*)\cdot \log_{t^*}\sqrt{\frac{\hat{\theta}}{\theta}}-1\right)$ \\ \hline\hline
    \end{tabular}

     \begin{tabular}{c|c|cc|cc}\hline\hline
      \acrotem & $\ve{\theta}_{\mathrm{l}}$ & $\ve{\theta}_{\mathrm{r}}$ \\ \hline
       1D $t$-exponential & $-\expect_i\left[\frac{1}{(-\theta_i)^{1-t^*}}\right] / \expect_i\left[\frac{1}{(-\theta_i)^{2-t^*}}\right]$ & $- \expect_i\left[(-\theta_i)^{2-t^*}\right]$\\
    1D $t$-Gaussian ($\mu = 0$)   & $-\expect_i\left[\frac{1}{(-\theta_i)^{\frac{1-t^*}{2}}}\right] / \expect_i\left[\frac{1}{(-\theta_i)^{\frac{3-t^*}{2}}}\right]$ & $-\frac{1}{(c_{t^*} \sqrt{t^*})^{1-t^*}}\cdot \expect_i\left[(-\theta_i)^{\frac{3-t^*}{2}}\right]$ \\ \hline\hline
    \end{tabular}
    \caption{Functions of key interest related to some \acrotem s families, mentioning the source ($\ve{\lambda}$), natural ($\ve{\theta}$) and expectation ($\ve{\hbar}$) parameters, the cumulant $G_t(\ve{\theta})$ and its convex dual $G^\star_t(\ve{\hbar})$, the corresponding divergence on natural parameters $B_{G_t}(\hat{\ve{\theta}} \| \ve{\theta})$ \eqref{eqConfB} and its two population minimizers.  Remark that for each of them $\alpha_*$ in Lemma \ref{lem-left-popmin} has a closed form and we obtain two \textit{different} generalisations of Itakura-Saito divergence with $B_{G_t}(\hat{\ve{\theta}} \| \ve{\theta})$. We let $t^{**}\defeq 2/(3-t^*)$, $c_t \defeq \sqrt{\frac{\pi}{1-t}} \frac{\Gamma\left(1+\frac{1}{1-t}\right)}{\Gamma\left(\frac{3}{2}+\frac{1}{1-t}\right)}$.}
    \label{tab:distributions}
  \end{table*}
  Table \ref{tab:distributions} presents a few examples of \acrotem s and the related parameters useful in our clustering context (see Sections \ref{sec-ITG} and \ref{sec-clu}). Hereafter, we assume $G_t$ strictly convex and differentiable.
  

%% file: content/clustering.tex
\section{An information theoretic/geometric result}\label{sec-ITG}

\acrotem s being a generalisation of exponential families, one would expect that the key information theoretic / information geometric identity \eqref{defCLUNP} does admit a generalisation to our context. This is indeed the case and we now derive it. For this to happen, we also need a generalisation of the KL divergence used in \eqref{defFDIV}. We thus define
\begin{eqnarray}
\label{eq-gen-FD}
F_{t}(\tilde{P}_{t|\hat{\ve{\theta}}}\|\tilde{P}_{t|\ve{\theta}}) & \defeq & \int f\left(\frac{\mathrm{d}\tilde{p}_{t|\hat{\ve{\theta}}}}{\mathrm{d}\xi} \oslash_t \frac{\mathrm{d}\tilde{p}_{t|\ve{\theta}}}{\mathrm{d}\xi} \right) \cdot \mathrm{d}\tilde{p}_{t|\ve{\theta}},\\
f & \defeq & -\log_t,\nonumber\\
x \oslash_t y & \defeq & (x^{1-t} - y^{1-t}+1)_+^{\frac{1}{1-t}} \text{ if } x,y \geq 0 \text{ else undefined}.\nonumber
\end{eqnarray}
We recover \eqref{defFDIV} as $t\rightarrow 1$. \eqref{eq-gen-FD} is equivalent to the \emph{tempered KL divergence} induced by the convex function~\eqref{eq:tsallis} that was introduced in~\citet{bitemp}. We now state our generalization of \eqref{defCLUNP}.
\begin{theorem}\label{thm-ITIG}
  For any 2 members of the same \acrotem~family,
  \begin{eqnarray*}
F_{t}(\tilde{P}_{t|\hat{\ve{\theta}}}\|\tilde{P}_{t|\ve{\theta}}) & = & B_{G_t}(\hat{\ve{\theta}}\|\ve{\theta}),
  \end{eqnarray*}
  where
  \begin{eqnarray}
B_{G_t}(\hat{\ve{\theta}}\|\ve{\theta}) \defeq \frac{G_t(\hat{\ve{\theta}}) - G_t({\ve{\theta}}) -(\hat{\ve{\theta}} - {\ve{\theta}})^\top\nabla G_t(\ve{\theta})  }{1+(1-t)G_t(\hat{\ve{\theta}})}. \label{eqConfB}
    \end{eqnarray}
  \end{theorem}
  (proof in \supplement, Section \ref{sec-proof-thm-ITIG}) One can see that the numerator in \eqref{eqConfB} is in fact the Bregman divergence with generator $G_t$. The whole construct $B_{G_t}$ belongs to a generalisation of Bregman divergences known as conformal Bregman divergences \citep{nnaOCD} and we recover Bregman divergences as $t\rightarrow 1$. Clustering with exponential families relies on the Bregman divergence as a distortion measure between parameters. In our case, the presence of the denominator $D_t(\hat{\ve{\theta}}) \defeq 1+(1-t)G_t(\hat{\ve{\theta}}) = \exp_t^{1-t}G_t(\hat{\ve{\theta}})$ is crucial for clustering if $\hat{\ve{\theta}}$ is an outlier, and some algebra allows to see that $D_t(\ve{\theta})$ is a function (increasing) proportional to the \textit{total mass} of a \acrotem~since this denominator also meets:
  \begin{eqnarray}
D_t(\hat{\ve{\theta}})^{\frac{1}{1-t^*}} & = & \int \exp_{t^*} \left(\frac{\hat{\ve{\theta}}^\top \ve{\phi}(\ve{x})}{t^*}\right) \mathrm{d}\xi, \label{eqDT}
  \end{eqnarray}
and the RHS is indeed proportional to $\mathrm{M}_{t^*}((1/t^*)\cdot {\ve{\theta}})$. In short, when $\hat{\ve{\theta}}$ in \eqref{eqConfB} is a data point, choosing a ``heavy'' enough \acrotem~in $D_t(\hat{\ve{\theta}})$ can have it grow sufficiently fast as $\hat{\ve{\theta}}$ moves far away and eventually reduce its influence on the cluster centroids. We now study clustering more formally.

  \section{Clustering and population minimizers}\label{sec-clu}
  
Let $\{\ve{\theta}_i\}_{i=1}^m$ be a training set of parameters endowed with an implicit (\textit{e.g.} uniform) distribution. We define two losses for the so-called left and right population minimizers:
 \begin{eqnarray}
L_{\mathrm{l}} (\ve{\theta}) \defeq \expect_i [B_{G_t}(\ve{\theta}\|\ve{\theta}_i)] \,\, ; \,\, L_{\mathrm{r}} (\ve{\theta}) \defeq \expect_i [B_{G_t}(\ve{\theta}_i\|\ve{\theta})]. \label{defLosses}
 \end{eqnarray}
The left and right population minimizers, respectively $\ve{\theta}_{\mathrm{l}}$ and $\ve{\theta}_{\mathrm{r}}$, are then defined as
  \begin{eqnarray}
\ve{\theta}_{\mathrm{l}} \defeq \arg\min_{\ve{\theta}} L_{\mathrm{l}} (\ve{\theta}) & ; & \ve{\theta}_{\mathrm{r}} \defeq \arg\min_{\ve{\theta}} L_{\mathrm{r}} (\ve{\theta}).\label{defpopmin}
  \end{eqnarray}
  The left and right population minimizers are the parameters whose corresponding losses are called Bregman information \citep[Section 3.1]{bmdgCWj}.
  We elaborate on clustering in two directions. The first is the elicitation of population minimizers and the second is their \textit{robustness} to outliers \citep{aIG,vlanTB}. To evaluate robustness, we add a new element $\ve{\theta}_*$ with weight $\epsilon$ in the new loss. The initial loss is scaled by $(1-\epsilon)$. The population minimizer is said robust to outliers if the new population minimizer satisfies $\ve{\theta}^{\mathrm{new}}_{\mathrm{l}/\mathrm{r}} - \ve{\theta}^{\mathrm{old}}_{\mathrm{l}/\mathrm{r}} = \epsilon \cdot \ve{z}(\ve{\theta}_*)$, where $\ve{z}(.)$, the influence function, has bounded norm.
  \paragraph{Population minimizers elicited} We first provide both population minimizers in \eqref{defpopmin}, reminding we assume $G_t$ strictly convex and differentiable. The simplest one is the right population minimizer.
  \begin{lemma}\label{lem-right-popmin}
    The right population minimizer \eqref{defpopmin} is given by
    \begin{eqnarray}
      \ve{\theta}_{\mathrm{r}} & = & \expect_i \left[\frac{1}{\exp_t^{1-t}(G_t(\ve{\theta}_i))}\cdot \ve{\theta}_i\right] .\label{rightpopmin}
      \end{eqnarray}
    \end{lemma}
    The proof of this Lemma trivially comes from \cite[Proposition 1]{bmdgCWj}, and it also recovers their result for Bregman divergences as $\lim_{t\rightarrow 1} \ve{\theta}_{\mathrm{r}} = \expect_i [\ve{\theta}_i]$. 
    We turn to the left population minimizer and let $\mathsf{T}_i(\ve{\theta}) \defeq G_t(\ve{\theta}_i) + (\ve{\theta} -\ve{\theta}_i)^\top \nabla G_t(\ve{\theta}_i)$ the value at $\ve{\theta}$ of the tangent hyperplane to $G_t$ at $\ve{\theta}_i$. We also let $N(\ve{\theta}) \defeq 1+(1-t) \expect_i[\mathsf{T}_i(\ve{\theta})]$.
\begin{lemma}\label{lem-left-popmin}
 The critical point of $L_{\mathrm{r}} (\ve{\theta})$ satisfies $\nabla G_t(\ve{\theta}_{\mathrm{l}}) = \alpha_* \cdot \expect_i \nabla G_t(\ve{\theta}_i)$ for some $\alpha_*>0$. It is the left population minimizer if $N(\ve{\theta}) > 0$.
\end{lemma}
The proof is given in \supplement, Section \ref{sec-proof-lem-left-popmin}, also shows that since $t\leq 1$,
\begin{eqnarray}
\alpha_* & \in & \left[1, \min_i \frac{1+(1-t) G_t(\ve{\theta}_i)}{N(\ve{\theta}_i)} \right], \label{balpha}
  \end{eqnarray}
  which provides a convenient initialisation interval for a line search of $\alpha_*$. Table \ref{tab:distributions} shows that it is also possible to get the left population minimizer in closed form for specific choices of \acrotem. One also sees that $\alpha \geq 1$ and $\lim_{t\rightarrow 1} \alpha = 1$, which gives us back the $f$-mean left population minimizer of Bregman divergences, noting also that $\lim_{t\rightarrow 1} N(\ve{\theta}) = 1$ so the condition $N(\ve{\theta}) > 0$ vanishes, and can in fact always be satisfied by choosing $t$ close enough to 1\footnote{In fact, Lemma \ref{lemNPOS} provided in \supplement~shows it is a weak assumption to directly assume $N(\ve{\theta}) > 0$.}. Notice also that the left population minimizer is unique.
    \paragraph{Robustness of population minimizers} We first tackle the right population minimizer: the average is notoriously not robust and so in the case of Bregman divergences, this population minimizer can never be robust, regardless of the divergence. In the case of \acrotem, however, the partition function gives a direct handle for robustness as the following simple Lemma shows, $\|.\|$ being any norm.
    \begin{lemma}\label{lem-rob-right}
If $G_t(\ve{\theta}) = \Omega(\|\ve{\theta}\|)$ and $t\neq 1$, the right population minimizer \eqref{defpopmin} is robust.
\end{lemma}
(proof in \supplement, Section \ref{sec-proof-lem-rob-right}) Obviously, this robustness property vanishes as $t\rightarrow 1$. Since the denominators in \eqref{rightpopmin} are an increasing function of \eqref{eqDT}, one roughly gets that robustness is achieved by picking a ``heavy'' enough \acrotem. The case of the left population minimizer is treated in the following Lemma. For any strictly convex $G$, the ``$f$-mean generated by $G$'' refers to $\nabla G^{-1}(\expect_i \nabla G(\ve{\theta}_i))$, which is the left population minimizer for exponential families \citep{bmdgCWj}.
  \begin{lemma}\label{lem-rob-left}
    Suppose $G_t$ strongly convex differentiable. Then the left population minimizer \eqref{defpopmin} is robust iff the $f$-mean generated by $G_t$ is robust.
\end{lemma}
(proof in \supplement, Section \ref{sec-proof-lem-rob-left}) A technical advantage of this Lemma is that to show the robustness of our left population minimizer, it is necessary and sufficient to investigate that of the $f$-mean, which can be simple to establish. As an example, the harmonic mean is robust, and it is the left population minimizer associated to the (1D) exponential distribution. In Table \ref{tab:distributions} for the 1D $t$-exponential \acrotem, one can check that $\theta_{\mathrm{l}}$ is also robust: suppose $\theta_j$ is the outlier. When $\theta_j \rightarrow -\infty$, its influence vanishes in $\theta_{\mathrm{l}}$ and when $\theta_j \rightarrow 0$, $\theta_{\mathrm{l}} \sim \theta_j \rightarrow 0$. Formal robustness is a binary notion but the experiments shall unveil that \textit{improved} robustness can also be achieved for $t< 1$ when the case $t=1$ is already robust.

Finally, there is an interesting parallel on robustness to be made between the left and right population minimizers. We have seen that the right population minimizer is robust if $G_t$ is chosen ``large enough". One can remark that $1+(1-t) G_t(\ve{\theta}_i) - N(\ve{\theta}_i) = (1-t) D_{G_t}(\ve{\theta}_i\|\ve{\theta}) \geq 0$, $D_{G_t}$ being a Bregman divergence. If $\ve{\theta}_i$ is an outlier, it may well be the case that $(1+(1-t) G_t(\ve{\theta}_i))/N(\ve{\theta}_i)$ becomes huge but the right bound in \eqref{balpha} depends on the $\min$ of the training sample's ratios and thus is that of a non-outlier. Thus, picking $G_t$ to get a robust right population minimizer does not \textit{a priori} prevent the left population minimizer from being robust \textit{as well}, a property that cannot hold for exponential, $q$-exponential nor deformed exponential families. 

%% file: content/experiments.tex
\section{Experiments}\label{sec-exp}

\setlength\tabcolsep{0pt}

\begin{figure}
  \centering
  \begin{tabular}{cc|cc}\hline\hline
    \multicolumn{2}{c|}{Left center} & \multicolumn{2}{c}{Right center} \\ 
    $t=0.0$ & $t=1.0$ & $t=0.0$ & $t=1.0$ \\ \hline
    \includegraphics[trim=0bp 0bp 0bp 0bp,clip,width=0.24\textwidth]{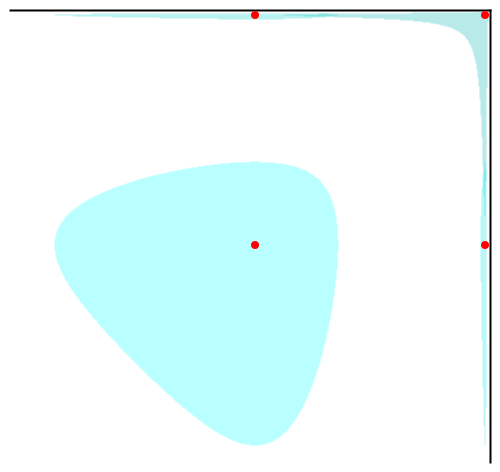} &  \includegraphics[trim=0bp 0bp 0bp 0bp,clip,width=0.24\textwidth]{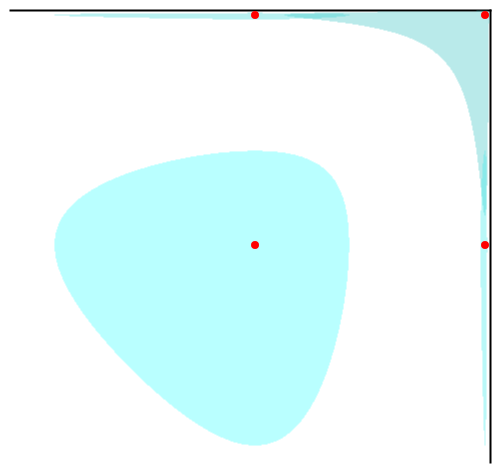} & \includegraphics[trim=0bp 0bp 0bp 0bp,clip,width=0.24\textwidth]{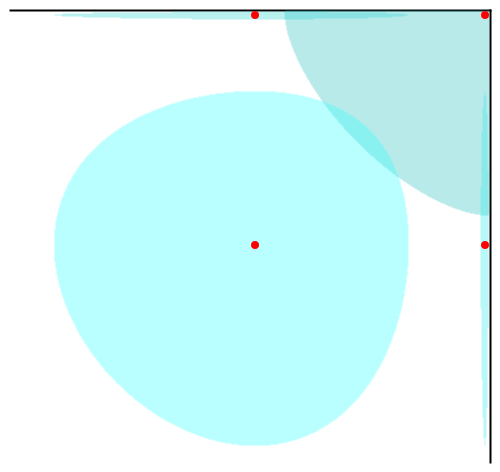} &  \includegraphics[trim=0bp 0bp 0bp 0bp,clip,width=0.24\textwidth]{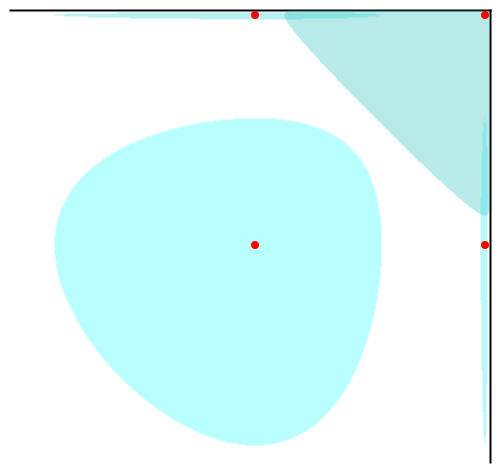} \\ \hline\hline
    \end{tabular}
\caption{Information geometric balls for the 1D $t$-exponential (domain = $\mathbb{R}_{-*}^2$); each plot displays four balls whose centers are the same among plots. Balls are computed so that the \textit{on-screen pixel} radius is fixed, so as not to get disproportionate balls between plots (see text for details).}
    \label{fig:balls}
  \end{figure}

We report experiments on simulated data on four topics related to clustering: (a) the shape of the balls whose associated distortion is $B_{G_t}$ in \eqref{eqConfB}, (b) Voronoi diagrams associated to the cluster centers, (c) robustness, and (d) clustering with or without noise. We focus our experiments on the divergence associated to the 1D $t$-exponential measure in Table \ref{tab:distributions}, which is a generalisation of the Itakura-Saito divergence. In a domain of dimension $>1$, the divergence we compute is just a sum of coordinate-wise 1D, scalar divergences, thereby mimicking a separable divergence, which is a common approach in ML.

  \setlength\tabcolsep{0pt}
  \newcommand{\vorosize}{0.188}
  
  \begin{figure}[t]
  \centering
  \begin{tabular}{c|ccccc}\hline\hline
    \multicolumn{6}{c}{Left center}\\\hline
    \rotatebox{90}{$t=0$} & \includegraphics[trim=0bp 0bp 0bp 0bp,clip,width=\vorosize\textwidth]{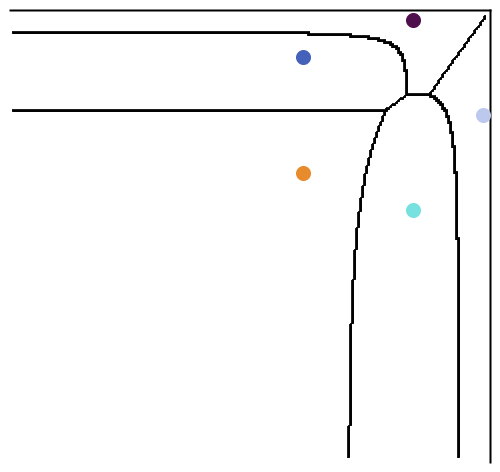} & \includegraphics[trim=0bp 0bp 0bp 0bp,clip,width=\vorosize\textwidth]{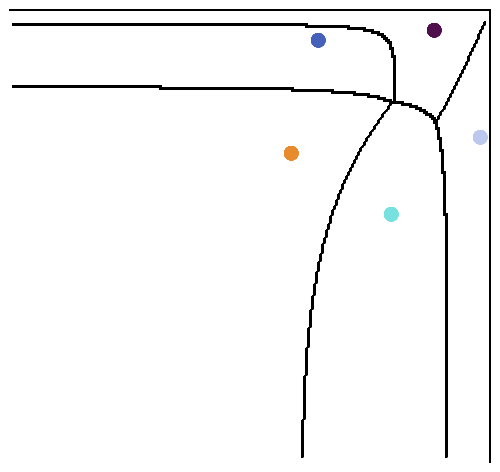} & \includegraphics[trim=0bp 0bp 0bp 0bp,clip,width=\vorosize\textwidth]{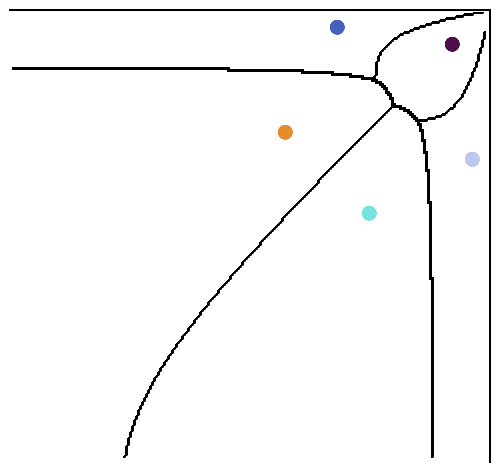} & \includegraphics[trim=0bp 0bp 0bp 0bp,clip,width=\vorosize\textwidth]{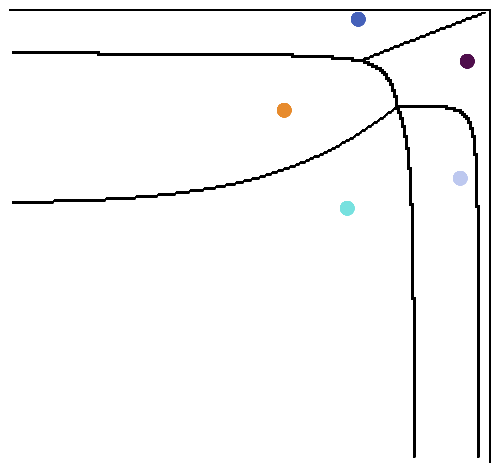} & \includegraphics[trim=0bp 0bp 0bp 0bp,clip,width=\vorosize\textwidth]{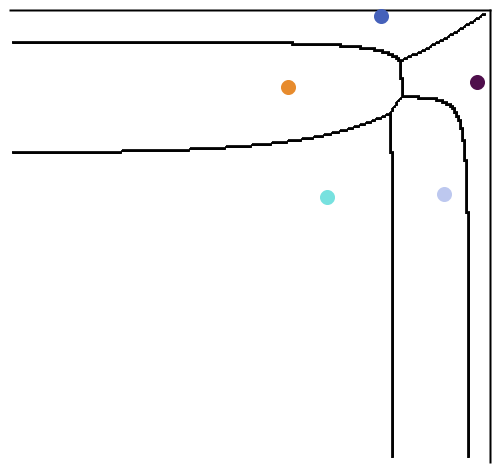} \\   \rotatebox{90}{$t=0.5$} & \includegraphics[trim=0bp 0bp 0bp 0bp,clip,width=\vorosize\textwidth]{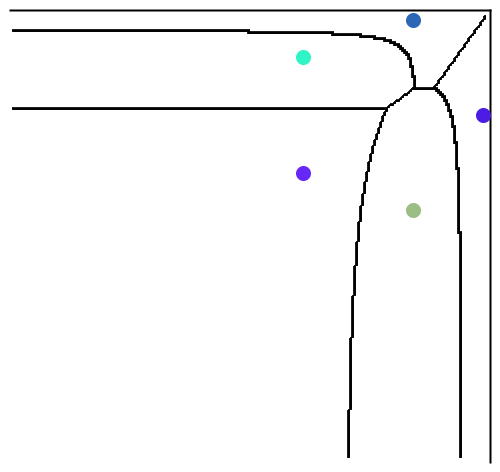} & \includegraphics[trim=0bp 0bp 0bp 0bp,clip,width=\vorosize\textwidth]{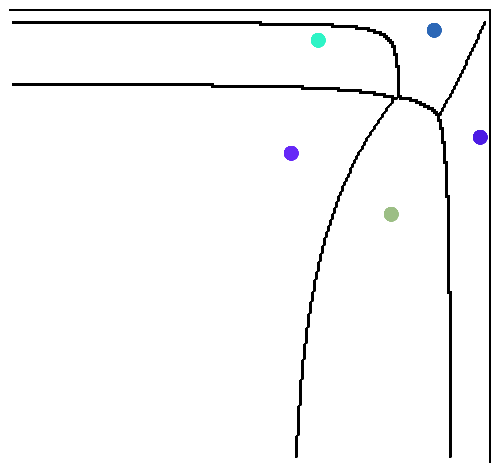} & \includegraphics[trim=0bp 0bp 0bp 0bp,clip,width=\vorosize\textwidth]{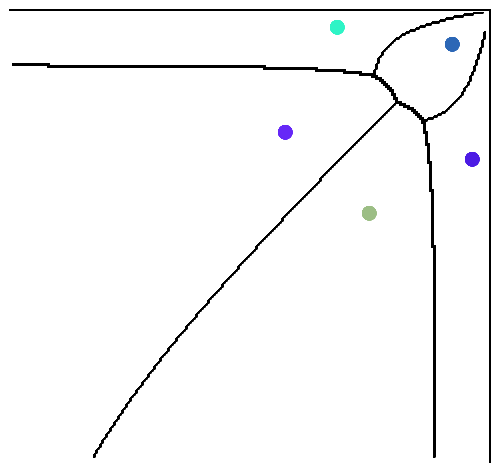} & \includegraphics[trim=0bp 0bp 0bp 0bp,clip,width=\vorosize\textwidth]{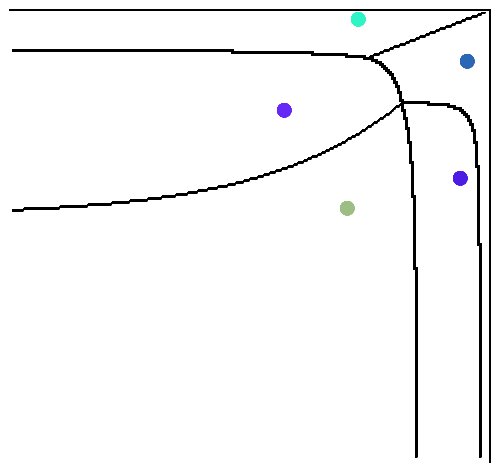} & \includegraphics[trim=0bp 0bp 0bp 0bp,clip,width=\vorosize\textwidth]{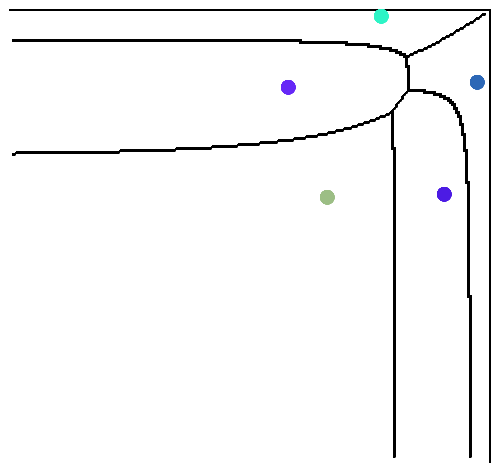} \\ \rotatebox{90}{$t=1.0$} & \includegraphics[trim=0bp 0bp 0bp 0bp,clip,width=\vorosize\textwidth]{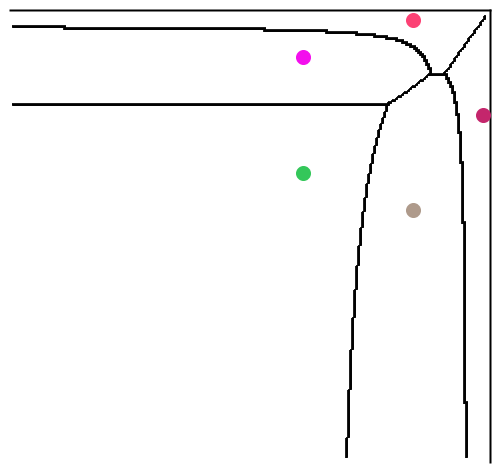} & \includegraphics[trim=0bp 0bp 0bp 0bp,clip,width=\vorosize\textwidth]{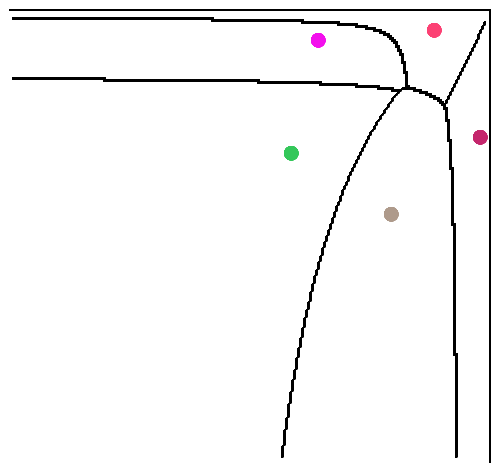} & \includegraphics[trim=0bp 0bp 0bp 0bp,clip,width=\vorosize\textwidth]{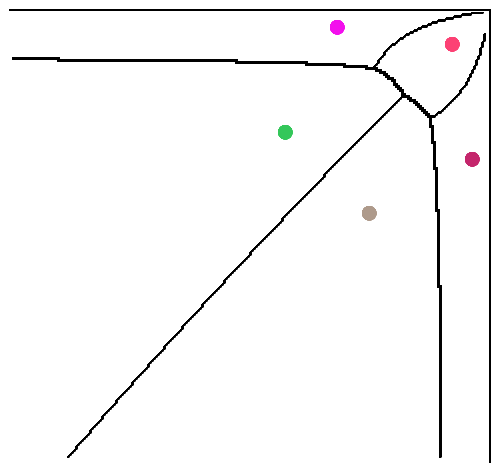} & \includegraphics[trim=0bp 0bp 0bp 0bp,clip,width=\vorosize\textwidth]{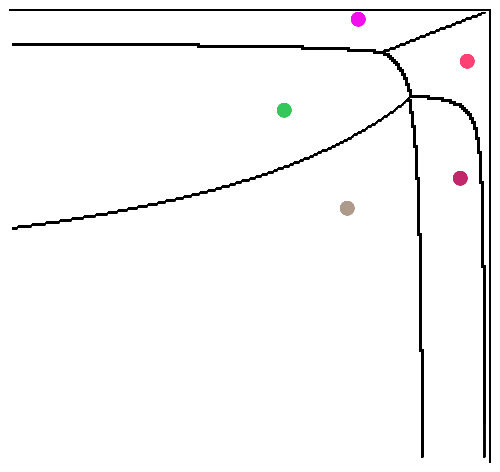} & \includegraphics[trim=0bp 0bp 0bp 0bp,clip,width=\vorosize\textwidth]{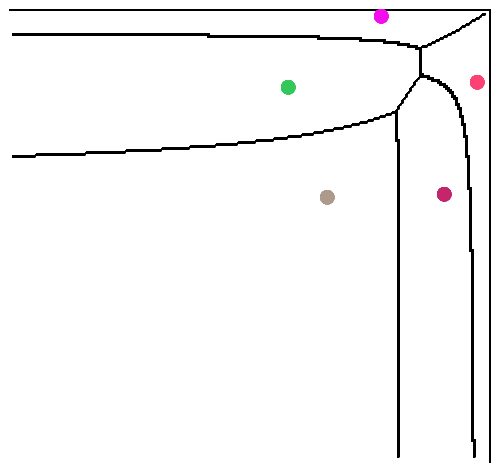} \\\Xhline{2pt}
       \multicolumn{6}{c}{Right center}\\\hline
    \rotatebox{90}{$t=0$} & \includegraphics[trim=0bp 0bp 0bp 0bp,clip,width=\vorosize\textwidth]{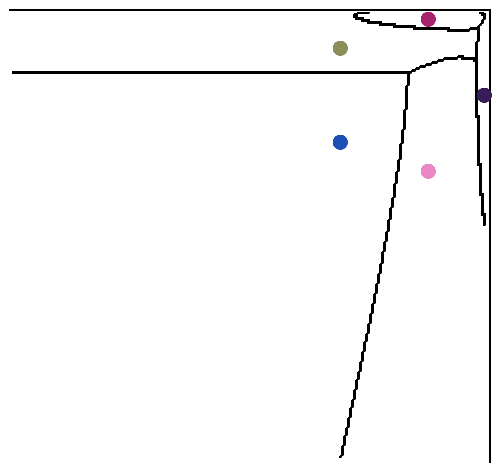} & \includegraphics[trim=0bp 0bp 0bp 0bp,clip,width=\vorosize\textwidth]{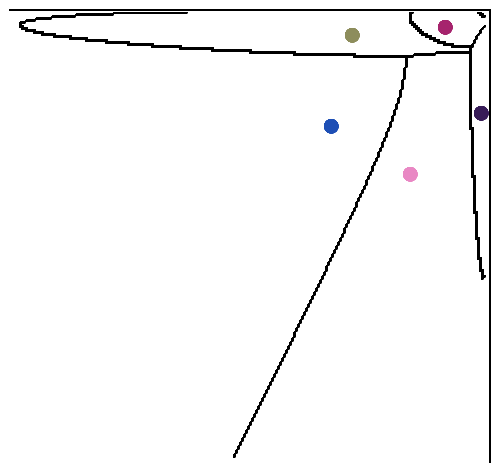} & \includegraphics[trim=0bp 0bp 0bp 0bp,clip,width=\vorosize\textwidth]{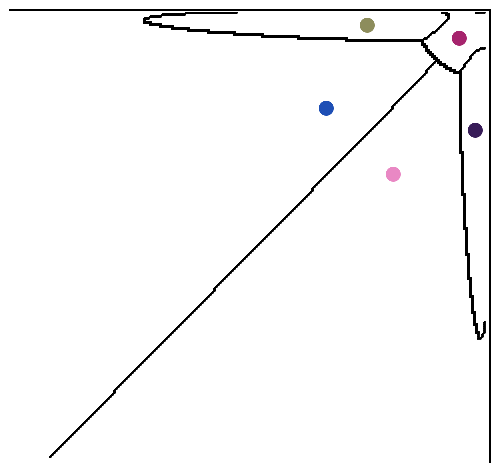} & \includegraphics[trim=0bp 0bp 0bp 0bp,clip,width=\vorosize\textwidth]{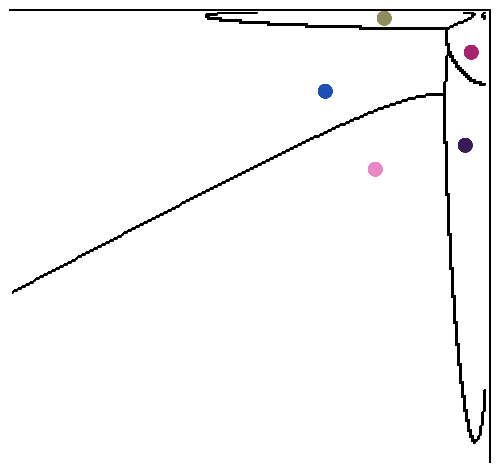} & \includegraphics[trim=0bp 0bp 0bp 0bp,clip,width=\vorosize\textwidth]{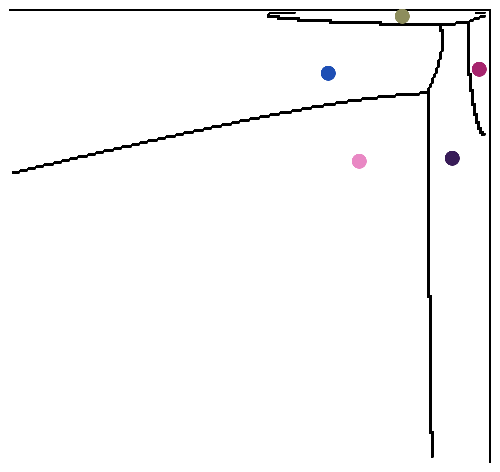} \\   \rotatebox{90}{$t=0.5$} & \includegraphics[trim=0bp 0bp 0bp 0bp,clip,width=\vorosize\textwidth]{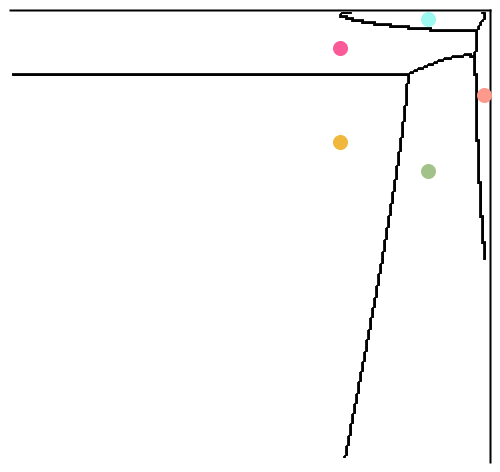} & \includegraphics[trim=0bp 0bp 0bp 0bp,clip,width=\vorosize\textwidth]{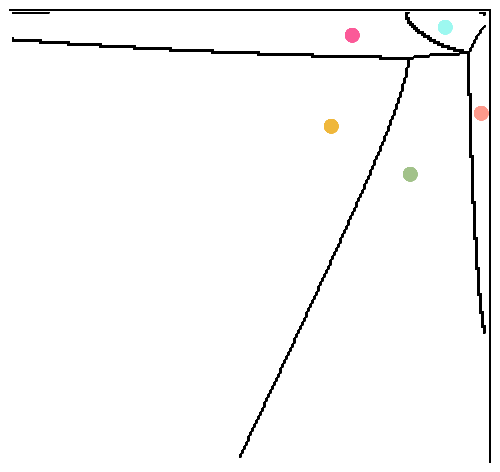} & \includegraphics[trim=0bp 0bp 0bp 0bp,clip,width=\vorosize\textwidth]{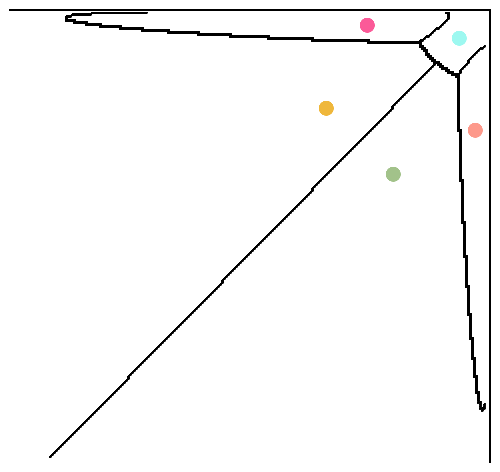} & \includegraphics[trim=0bp 0bp 0bp 0bp,clip,width=\vorosize\textwidth]{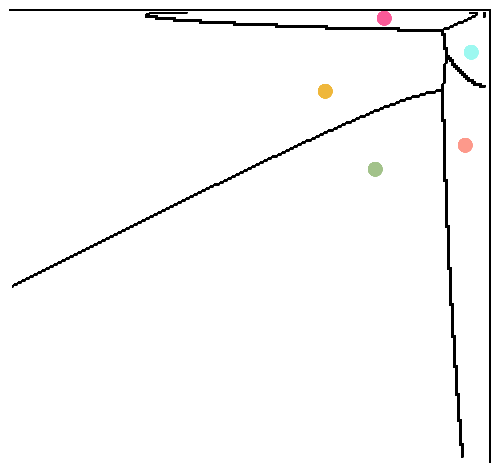} & \includegraphics[trim=0bp 0bp 0bp 0bp,clip,width=\vorosize\textwidth]{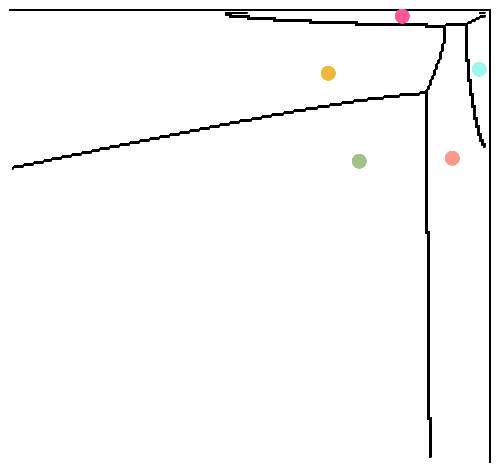} \\ \rotatebox{90}{$t=1.0$} & \includegraphics[trim=0bp 0bp 0bp 0bp,clip,width=\vorosize\textwidth]{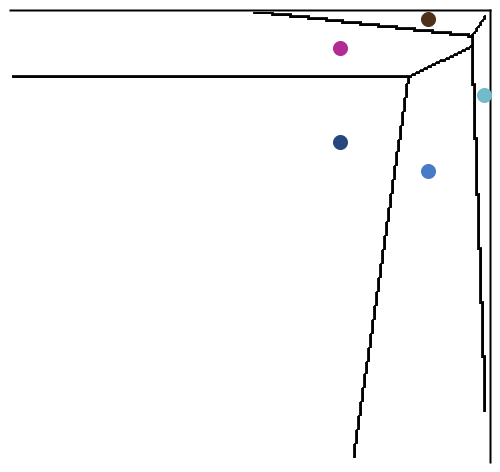} & \includegraphics[trim=0bp 0bp 0bp 0bp,clip,width=\vorosize\textwidth]{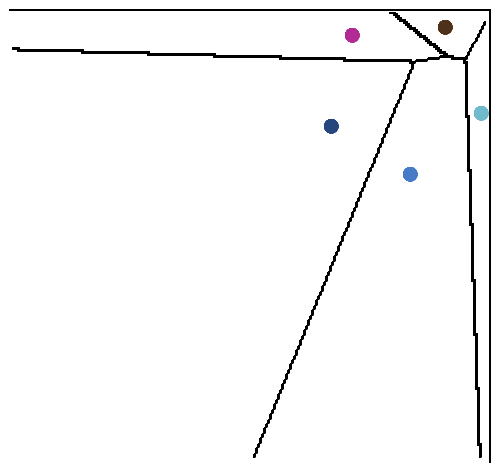} & \includegraphics[trim=0bp 0bp 0bp 0bp,clip,width=\vorosize\textwidth]{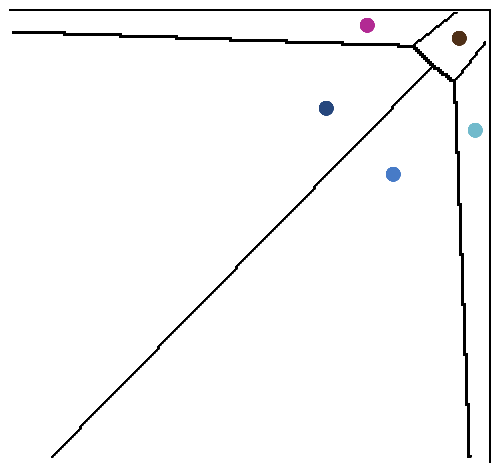} & \includegraphics[trim=0bp 0bp 0bp 0bp,clip,width=\vorosize\textwidth]{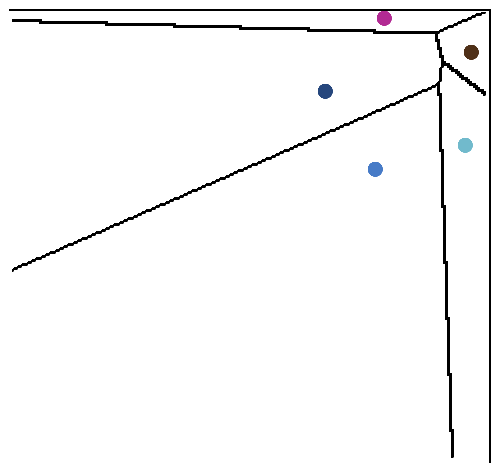} & \includegraphics[trim=0bp 0bp 0bp 0bp,clip,width=\vorosize\textwidth]{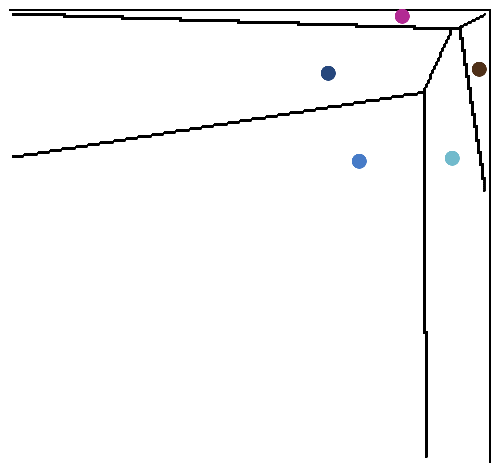} \\ \hline\hline
    \end{tabular}
\caption{Voronoi diagrams associated to the left (top) and right (bottom) center of the $B_{G_t}$ divergence of the 1D $t$-exponential (domain = $\mathbb{R}_{-*}^2$), where cell centers are the vertices of a rotating regular pentagon, for $t \in \{0, 0.5, 1\}$.}
    \label{fig:voronoi-leftright}
  \end{figure}

\paragraph{Shape of the information-geometric Balls} An important question for clustering, especially when it comes to generalizing approaches based on Bregman divergences, is the shape of the corresponding information-geometric balls, \textit{i.e.}, balls defined by a radius, a distortion and a center, generalizing the classical Euclidean balls whose distortion, the squared Euclidean distance, is a particular case of Mahalanobis divergence. Generalized to Bregman divergences, the balls can adopt a variety of shapes, even becoming eventually non-convex when the center is on the left position of the Bregman divergence \citep{nlkMB}. In our case, Figure~\ref{fig:balls} shows examples of balls for the 1D $t$-exponential \acrotem, thus generalizing the Itakura-Saito balls (they appear for $t=1$). One can remark that extending $t<1$ allows for more ``extreme'' shapes, where balls are more ``flattened", in particular when they are close to the quadrant's border (left center) or more ``round'' for the right center. Having increased diversity in ball shapes is good for clustering.

\paragraph{Voronoi diagrams} Another important practical question is the shape of Voronoi diagrams that partition the space in cells associated to a training data point being the closest center, and thus define the boundaries of clusters in a clustering. Since the information geometric divergences (Bregman divergences or our $B_{G_t}$ in \eqref{eqConfB}) are not symmetric in general, we have two types of Voronoi diagrams, a left and a right one depending on whether the cell's center is put in the left or right position in the corresponding divergence. There is a big difference between Voronoi diagrams associated to Bregman divergences \citep{bnnBV} and those associated with $B_{G_t}$ in \eqref{eqConfB}: the right Voronoi diagram is always affine with convex polyhedral cells for all Bregman divergences. In our case, this does not hold anymore and thus, we end up with two curved Voronoi diagrams.

  \newcommand{\outsize}{0.18}
\setlength\tabcolsep{0pt}
  \begin{figure}
  \centering
  \begin{tabular}{c|ccccc}\hline\hline
   \rotatebox{90}{$t=0.0$} &   \includegraphics[trim=0bp 0bp 0bp 0bp,clip,width=\outsize\textwidth]{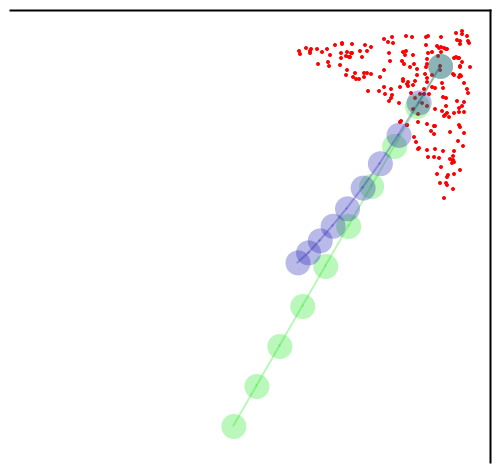} &      \includegraphics[trim=0bp 0bp 0bp 0bp,clip,width=\outsize\textwidth]{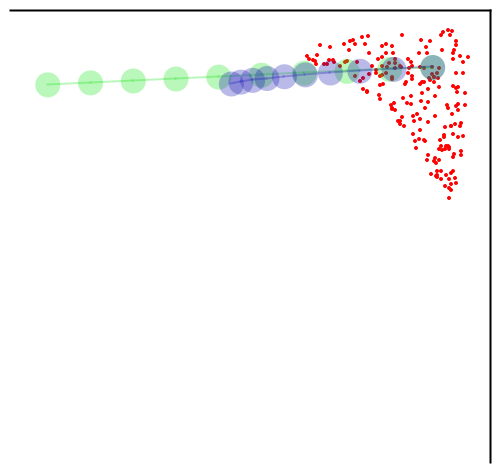} &     \includegraphics[trim=0bp 0bp 0bp 0bp,clip,width=\outsize\textwidth]{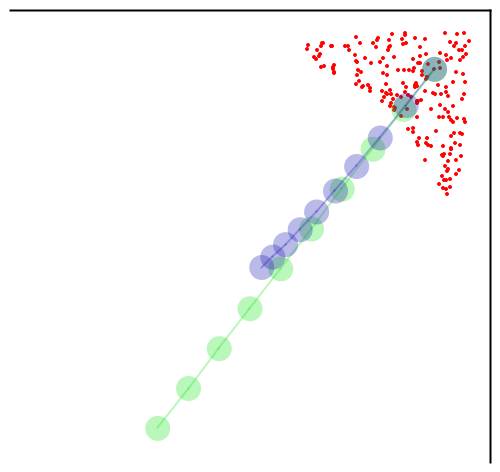} &      \includegraphics[trim=0bp 0bp 0bp 0bp,clip,width=\outsize\textwidth]{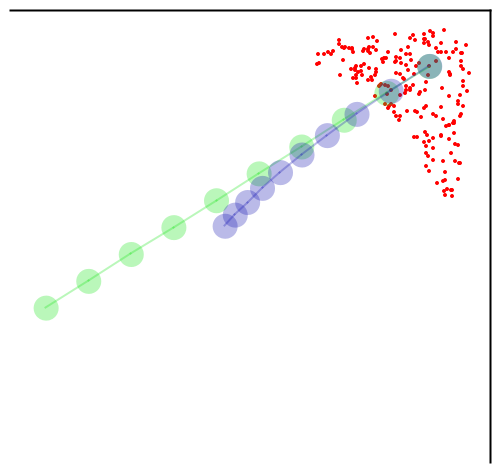} &      \includegraphics[trim=0bp 0bp 0bp 0bp,clip,width=\outsize\textwidth]{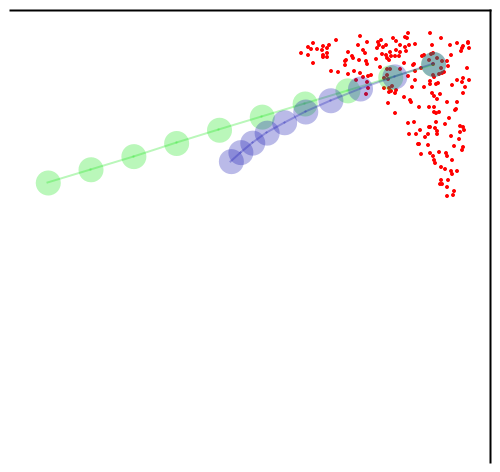} \\ \hline
  \rotatebox{90}{$t=1.0$} & \includegraphics[trim=0bp 0bp 0bp 0bp,clip,width=\outsize\textwidth]{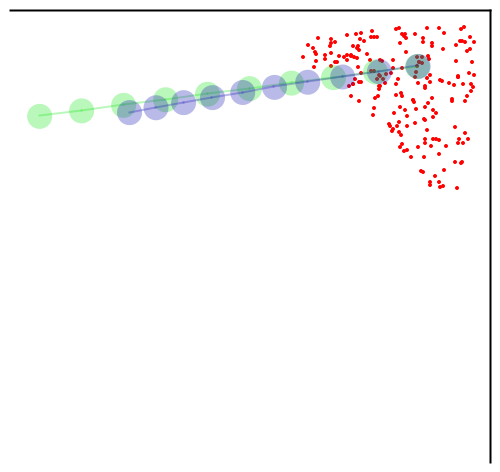} &      \includegraphics[trim=0bp 0bp 0bp 0bp,clip,width=\outsize\textwidth]{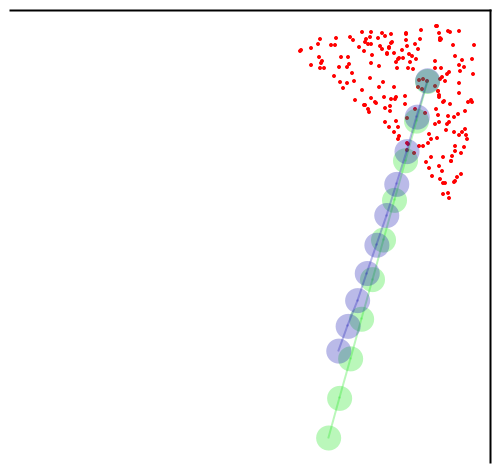} &     \includegraphics[trim=0bp 0bp 0bp 0bp,clip,width=\outsize\textwidth]{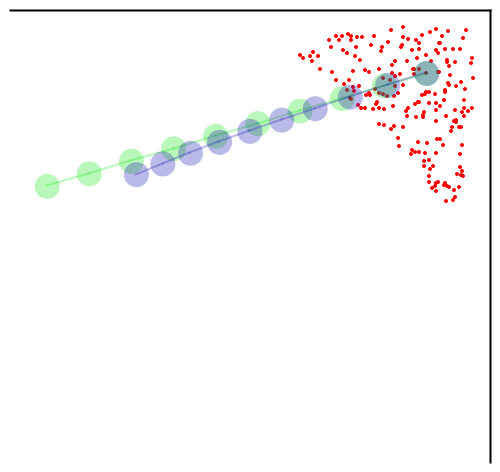} &      \includegraphics[trim=0bp 0bp 0bp 0bp,clip,width=\outsize\textwidth]{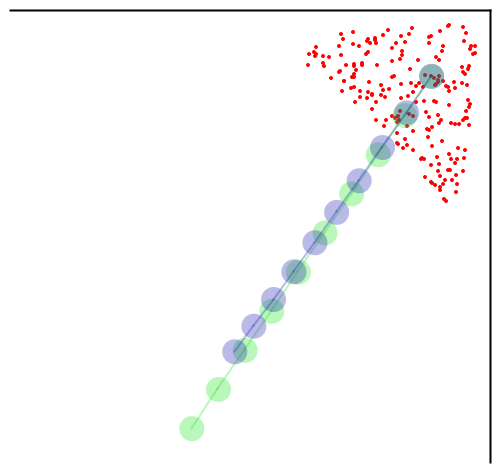} &      \includegraphics[trim=0bp 0bp 0bp 0bp,clip,width=\outsize\textwidth]{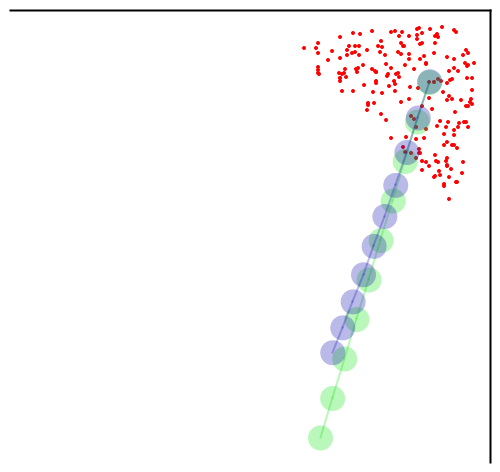} \\ \hline\hline
    \end{tabular}
\caption{Outlier effect on the left population minimizer of the 1D $t$-exponential (domain = $\mathbb{R}_{-*}^2$) for two values ot $t$, on 5 random trials (columns). Clusters are generated by uniformly sampling 200 points in a $B_F$-ball whose \textit{on-screen pixel} radius is fixed. A point close to the left population minimizer is chosen and treated as an outlier (green), with weight $5000 \times$ that of non-outliers. We then move away the outlier with a fixed step size (in green) and compute the resulting cluster center (in blue, see text for details).}
    \label{fig:outliers}
  \end{figure}

  \paragraph{Robustness} To analyze whether we can indeed observe improved robustness for $t\neq 1$ vs. $t=1$, we have used the 1D $t$-exponential's left population minimizer. It is an interesting case because for $t=1$, the divergence is Itakura-Saito divergence and its left population minimizer, the harmonic mean, is robust to outliers (See Section \ref{sec-clu}). Whether we can get improved robustness for $t\neq 1$ is displayed in Figure~\ref{fig:outliers}. Here, we chose a point close to the average, that we associated with a very heavy weight and then move away progressively the point by a constant vector in $\mathbb{R}^2$. The resulting trajectory of the outlier, in green, is picked at random. We then compute the trajectory of the population minimizer, in blue. One can observe that for $t=1$, the center moves away with a segment length slowly decreasing, whereas, for $t=0$, this length quickly decreases as the outlier moves far away, displaying improved robustness. Note that the robustness for $t=1$ appears more clearly for a displacement of the outlier further away, which is not shown to keep the pictures readable. Another interesting phenomenon appears, not just from the standpoint of the distance of the new center to its original position, but also from the standpoint of the \textit{angle} to its original position, as measured by a cone whose half lines cross the origin and go through the two centers (before and after max displacement): one can check that this angle is smaller for the $t=0$ case. Equivalently, for $t=1$, the center is not just dragged away according to a distance that is larger than for $t=0$: it also follows more closely the trajectory of the outlier compared to $t=0$.

  \newcolumntype{?}{!{\vrule width 2pt}}
  \newcommand{\pnoise}{p_{\mbox{\tiny{noise}}}}
  \newcommand{\error}{p_{\mbox{\tiny{err}}}}
  \newcommand{\csplit}{p_{\mbox{\tiny{split}}}}
  \newcommand{\divergence}{\overline{B}_F}

  \setlength\tabcolsep{2pt}
  \begin{table*}
    \centering
    {\footnotesize
  \begin{tabular}{m{1em}|c|c|cc:c?c|cc:c|c}\hline\hline
  \multirow{8}{*}{\rotatebox{90}{$\pnoise = 0$}} &  close- & & \multicolumn{3}{c?}{Clustering} &  far-   & \multicolumn{3}{c|}{Clustering} & \multirow{8}{*}{\begin{overpic}[trim=0bp 0bp 0bp 0bp,clip,width=0.18\textwidth]{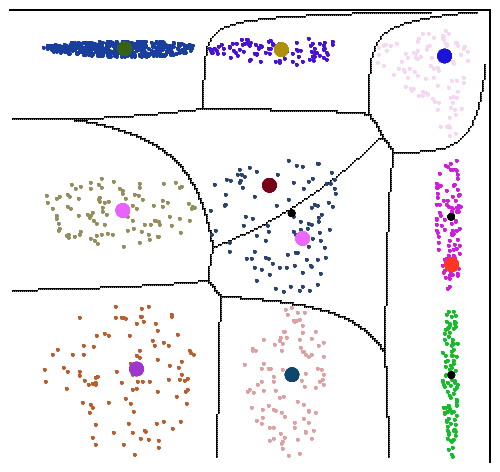}
 \put (1,5) {$(t = 0)$}
\end{overpic}} \\
    & 3$\times$3 &  & $t=0$ & $t=0.5$ & $t=1.0$ &     3$\times$3 & $t=0$ & $t=0.5$ & $t=1.0$ & \\ \cline{2-10}
   &  \multirow{3}{*}{$\mathcal{B}_{t=0}$} & $\error$ & 0.43$\pm$0.29 & 0.43$\pm$0.29 & \underline{0.35$\pm$0.24} & \multirow{3}{*}{$\mathcal{B}_{t=0}$} & 0.57$\pm$0.29 & \underline{0.52$\pm$0.32} & 0.61$\pm$0.30 &  \\
    & & $\csplit$  & 0.22$\pm$0.09 & 0.21$\pm$0.09 & \textbf{0.17$\pm$0.08} &  & 0.17$\pm$0.07 & \textbf{0.13$\pm$0.06}  & 0.17$\pm$0.06 & \\
   &  &  $\divergence$ & 2.00$\pm$0.82 & \underline{1.95$\pm$1.03} & 2.29$\pm$1.11  &  & 3.22$\pm$1.80  & \textbf{2.93$\pm$1.53} & 3.77$\pm$1.76  & \\ \cline{2-10}
   & \multirow{3}{*}{$\mathcal{B}_{t=1}$} & $\error$  &  \underline{0.40$\pm$0.25} & 0.44$\pm$0.29 & 0.46$\pm$0.28  & \multirow{3}{*}{$\mathcal{B}_{t=1}$} &  0.62$\pm$0.30  & 0.52$\pm$0.30 &  \underline{0.49$\pm$0.33} &  \\
   &  & $\csplit$ &  \underline{0.19$\pm$0.09} & 0.22$\pm$0.10 & 0.21$\pm$0.09 & & 0.16$\pm$0.07 & 0.16$\pm$0.06 & 0.16$\pm$0.08 & \\
                                                 &  & $\divergence$ & 2.03$\pm$1.21  & 2.11$\pm$1.10 & \underline{1.86$\pm$1.19} & & 2.42$\pm$1.15 & \underline{2.21$\pm$1.08} & 2.25$\pm$0.95 & \\ \Xhline{2pt}
   \multirow{8}{*}{\rotatebox{90}{$\pnoise = 0.1$}} & close-  &        & \multicolumn{3}{c?}{Clustering} &  far-   & \multicolumn{3}{c|}{Clustering} & \multirow{8}{*}{\begin{overpic}[trim=0bp 0bp 0bp 0bp,clip,width=0.18\textwidth]{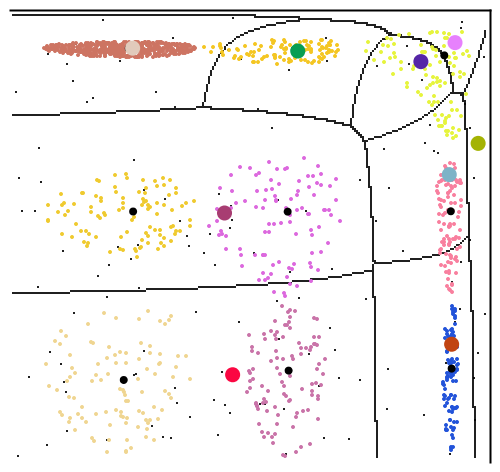}
 \put (1,5) {$(t = 1)$}
\end{overpic}} \\
  &  3$\times$3 &  & $t=0$ & $t=0.5$ & $t=1.0$ &     3$\times$3 & $t=0$ & $t=0.5$ & $t=1.0$ & \\ \cline{2-10}
   & \multirow{3}{*}{$\mathcal{B}_{t=0}$} & $\error$ & \textbf{0.28$\pm$0.11} & 0.30$\pm$0.12 & 0.32$\pm$0.13 & \multirow{3}{*}{$\mathcal{B}_{t=0}$} & \textbf{0.24$\pm$0.16} & 0.29$\pm$0.22& 0.37$\pm$0.25 &  \\
    & & $\csplit$  & \textbf{0.10$\pm$0.08} & 0.14$\pm$0.05 & 0.15$\pm$0.07 &  & 0.06$\pm$0.06 & \textbf{0.05$\pm$0.06}  & 0.08$\pm$0.07 & \\
   &  &  $\divergence$ & 14.86$\pm$11.30 & 8.76$\pm$5.35 & \textbf{4.77$\pm$4.12}  &  & \textbf{6.57$\pm$5.61}  & 10.89$\pm$5.36 & 12.18$\pm$9.52  & \\ \cline{2-10}
   & \multirow{3}{*}{$\mathcal{B}_{t=1}$} & $\error$  & 0.35$\pm$0.19 & 0.29$\pm$0.11 & \underline{0.28$\pm$0.12} & \multirow{3}{*}{$\mathcal{B}_{t=1}$} &  \underline{0.17$\pm$0.11}  & 0.38$\pm$0.19 &  0.37$\pm$0.20 &  \\
    & & $\csplit$ &  0.15$\pm$0.09 & \textbf{0.09$\pm$0.07} & 0.13$\pm$0.07 & & \textbf{0.03$\pm$0.06} & 0.06$\pm$0.06 & 0.13$\pm$0.06 & \\
   & & $\divergence$ & \underline{2.05$\pm$1.20}  & 5.04$\pm$2.37 & 2.26$\pm$1.84 & & \underline{4.26$\pm$2.15} & 5.27$\pm$1.96 & \underline{4.26$\pm$1.76} & \\ \hline\hline
  \end{tabular}
  }
  \caption{Clustering with the left population minimizer of 1D $t$-exponential distributions and the results of the corresponding clusterings for $t\in \{0,0.5,1\}$ in the form average$\pm$std-dev (average over 50 runs), without (top table) and with noise ($10\%$, bottom table). \underline{Underlined} values are the best among the three $t$ choices and \textbf{bold faces} denote a significant winner in $t\in \{0,0.5\}$ (best result) vs. $t = 1.0$ using a Gaussian test, p-val=$.05$. Pictures on the right give an example result on far-3$\times$3 for $\mathcal{B}_{t=0}$, (clustering's $t$ value indicated, true clusters shown using random colors with bigger black dots as their centers, Voronoi diagram displayed; learned centers in big coloured dots, see text).}
    \label{clu:balls-33-n0}
  \end{table*}

  \paragraph{Clustering with and without noise} In this experiment, we test whether improved robustness can be translated to a better handling of noise. We treat noise as follows. We generate a fixed number of $k$ clusters (and keep this value for clustering). The clusters are generated by random sampling in information-geometric balls with the left center, for $t\in \{0,1\}$. Those balls are either close or far from each other. To make the clusters unbalanced, one cluster has 20$\times$ more points than the others. We then cluster using the left population minimizer using $k$-means type iterations (computing centers, reallocating points to clusters), and measure several metrics to assess the quality of clustering (see below). When there is noise, we generate it uniformly on the picture as an additional cluster. Noise thus biases clustering results \textit{but} it is not taken into account for the measurement of the metrics. To explain it better, we compute three metrics: (I) at the end of clustering, we compute a distortion between the true clusters centers and those found, \textit{excluding} the center of the noise cluster, using the following algorithm: we repeatedly compute the couple (theoretical center, learned center) that minimizes the average $B_{G_t}$ divergence (where we permute the roles of the centers), and remove the theoretical center from the list -- and eventually remove the learned center if there still exist learned centers (sometimes, clustering comes up with less than $k$ clusters). We finally compute the average of those distances and report it as ``$\divergence$"; (II) we compute the proportion of true clusters being split among learned clusters in such a way that less than $2/3$rd of the cluster belongs to a single learned clusters (we call these ``true clusters that are split"). We do not use a larger proportion than $.67$ to authorize some of the learned clusters to scrap a minor proportion of the true clusters; we report this proportion as ``$\csplit$". Of course, we do not count the noise cluster in this computation; finally (III) for each true cluster, we compute the learned cluster with the largest fraction of the true cluster and count the remaining proportion of the true cluster as an error term; we compute the average of those error over true clusters and denominate it as ``$\error$". Table \ref{clu:balls-33-n0} summarizes the results obtained, where each statistics is computed over 50 runs, along with example clusterings. Modulo the fact that we treat our theoretical clusters as the ground truth for clustering (there could be some slight changes in optimal clusterings, especially in the ``close'' configuration), Table \ref{clu:balls-33-n0} confirms that choices $t\neq 1$ can improve clustering, especially when there is noise, from the standpoint of all metrics.

%% file: content/discussion.tex
\section{Discussion}\label{sec-disc}

We split this discussion in three parts, from a focus on clustering to more general considerations on \acrotem s. 

On clustering, it is important to remark that one can always design heuristic / \textit{ad hoc} non-constant weighting for the clustering problem to artificially change its properties. Our approach is, we believe, the first that formally grounds such weights (Lemma \ref{lem-right-popmin}) in a principled approach to the measures whose parameters are clustered in disguise. Second, one may remark that (scalar) $f$-means have intuitive properties, such as monotonicity (increasing an argument cannot decrease the mean), idempotence (the mean of the same repeated value is the value itself) and bounding (the mean is in between the min and max argument values). Our population minimizers \textit{can} break these properties (unless $t=1$): for example, the left population average of the 1D $t$-exponential \acrotem~in Table \ref{tab:distributions} is monotonic and idempotent but does not meet the bounding constraint. Relaxing the constraints of the population minimizers outside those met by traditional means is not a bad thing, as ultimately the properties of a population minimizer depend on the distortion it is supposed to minimize in expectation. Also, as exemplified by our experiments, relaxing those properties can be beneficial. Ultimately, it can be a design choice to consider or tune \textit{ex ante}: for example, assuming $t\in [0,1]$, one needs $G_t(\min_i \theta_i) \leq 0, G_t(\max_i \theta_i) \geq 0$ to get bounding. One also has to keep in mind that clustering faces substantial impediments in terms of design choices \citep{kAI}. Third, our experiments have made use of simple random (Forgy) initialization for the cluster centers. A much more powerful initialisation with guarantees has been designed for clustering with Gaussians \citep{avKM} and extended to exponential families \citep{nlkMB}, and even to distortion classes without closed form for the population minimizers \citep{nnTJ}. While we deliberately did not rely on more sophisticated initialization to not bias divergences' influence using a purely random start clustering, it is a promising direction to investigate.

Second, as we noted in Section \ref{sec-rel}, some previous work related to robust clustering has also put a focus on links with distributions, departing from both Bregman divergences and exponential families \citep{lvanSR,vlanTB}. In our case, our generalisations of exponential families to \acrotem s, which allows for improved robustness as $t\neq 1$, still comes with a guarantee of "closedness" to exponential families. We provide a proof on a key parameter: the cumulant \eqref{eqGT} and show that one can always come "as close as desired" from the exponential family case with $t\neq 1$. Such a result is relevant not just to numerical analysis at large: the cumulant is indeed the ID of a family of distributions in exponential families and it is not available in closed form for classical generalisations of exponential families that are $q$-exponential families or deformed exponential families. We let $\Theta$ denote the (open) set of natural parameters.
\begin{theorem}\label{th-conti}
  $\forall \ve{\theta} \in \Theta, \forall \varepsilon > 0, \exists t<1 : |G_t(\ve{\theta}) - G_1(\ve{\theta})| \leq \varepsilon$.
\end{theorem}
(Proof in \supplement, Section \ref{sec-proof-th-conti}) As a consequence, we also get continuity in the neighborhood of the exponential family's case of the total mass of the \acrotem~(Lemma \ref{lemMf}) and of the convex conjugate of the cumulant.

Last, from a more general standpoint on \acrotem s, our general approach may seem close to the design of $q$-exponential families and even deformed exponential families -- the knowledgeable reader will notice that our \acroted s technically look similar to escort distributions in the way we design them through \eqref{defPTILDE}, despite a normalization which belongs to the \textit{divisive} normalisation of distribution rather than the subtractive normalisation of $q$-exponential families and deformed exponential families \citep{zwLD} (alternatively, we rely on a $t$-subtractive normalisation, using the arithmetic of \citet{nlwGA}). Classical escort distributions, however, appear independently of the $q$-exponential families or deformed exponential families: they do not belong to their axiomatization. In our case, they do, and the fact that we chose to somehow ``mix'' \acrotem~and \acroted~in the axiomatisation of the \acrotem, by constraining the normalization of the \acroted, seems to yield technical conveniences not known for $q$-exponential families or even deformed exponential families, the first of which is the elegant closed form of the cumulant in \eqref{eqGT}. Beyond such technical conveniences appear some concrete advantages for clustering. Given the ubiquity of exponential families and Bregman divergences in ML, many interesting questions arise on other potential advantages for other uses and applications. One promising direction is the investigation of the Riemannian geometry of the parameter space \citep{aIG,anMO}. 

%% file: content/conclusion.tex
\section{Conclusion}\label{sec-conc}

In this paper, we introduce a new generalisation of exponential families named tempered exponential measures, whose constrained maximum entropy design involves normalizing a dual instead of the measure itself as in the state-of-the-art generalisation of exponential families ($q$-exponential families and deformed exponential families). Tempered exponential measures provide a generalisation of Bregman divergences in the parameter space, which allows designing clustering with improved robustness properties compared to the classical $k$-means extended to exponential, $q$-exponential, or deformed exponential families.

Given the wide footprint of exponential families and Bregman divergences in ML and the fact that tempered exponential measures also provide new and general technical conveniences beyond the realm of clustering, more ML applications of this new tool are expected as well as additional technical insights relevant to ML such as the information geometry of the parameter space.

%% file: content/acknowledgments.tex
\section*{Acknowledgments}

The authors warmly thank Frank Nielsen for numerous remarks on this material.

%% file: content/appendix-proofs.tex
\section{Cheatsheet for $t$-functions, $t$-algebra and related functions}\label{sec-cheatsheet}

\paragraph{$t$-algebra} Following \cite{nlwGA}, we define
\begin{eqnarray}
  x \oplus_t y & \defeq & \log_t(\exp_t(x) \exp_t(y)) = x + y + (1-t)xy,\\
  x \ominus_t y & \defeq & \log_t\frac{\exp_t(x)}{\exp_t(y)} = \frac{x-y}{1+(1-t)y}, \\
  x \otimes_t y & \defeq & \exp_t(\log_t(x) + \log_t(y)) = \left(x^{1-t} + y^{1-t}-1\right)_+^{\frac{1}{1-t}} \mbox{ if } x,y \geq 0 \mbox{ else undefined},\\
  x \oslash_t y & \defeq & \exp_t(\log_t(x) - \log_t(y)) = \left(x^{1-t} - y^{1-t}+1\right)_+^{\frac{1}{1-t}} \mbox{ if } x,y \geq 0 \mbox{ else undefined}.
\end{eqnarray}

\paragraph{$t$-functions} $\log_t$ and $\exp_t$ satisfy
\begin{eqnarray}
  \log_t '(z) & = & z^{-t},\\
  \exp_t '(z) & = & \exp_t^t(z),\\
  {(\log_t)^*} '(z) & = & \left(\frac{z}{t^*}\right)^{-t^*},\\
  {(\exp_t)^*} '(z) & = & \exp_t(z).
\end{eqnarray}
For non-negative scalars $x, y \geq 0$, we also have
\begin{equation}
\begin{split}
    \label{eq:log_t_a/b}
  \log_t x\,y = \log_t x + x^{1-t}\log_t y\,,
  \\  \log_t\frac{x}{y} = \log_t x - (\frac{x}{y})^{1-t}\, \log_t y\, .
\end{split}
\end{equation}
\paragraph{General properties} The $t$-functions and $t$-algebra have the interesting property that properties of the $t=1$ functions transfer modulo the general rule that "classical arithmetic outside the function becomes $t$-arithmetic inside and vice-versa". For example:
\begin{eqnarray}
  \frac{\exp_t(x)}{\exp_t(y)} & = & \exp_t(x\ominus_t y),\\
  \exp_t(x) \oslash_t \exp_t (y) & = & \exp_t(x-y).
\end{eqnarray}
The $t$-functions also satisfy
\begin{eqnarray}
  \exp_{\frac{1}{t^*}}(z) & = & \frac{1}{\exp_t(-z)},\\
  \log_{\frac{1}{t^*}} z &=& -\log_t \frac{1}{z} 
\end{eqnarray}
$(\exp_t)^*(z)$ and $(\log_t)^*(z)$ are inverses of each other.
  \section{Proof of Theorem \ref{thTEXPM}}\label{sec-proof-thTEXPM}

We first show the expression of $\tilde{p}_{t|\theta}$ (in the scalar case for natural parameters for readability); the proof is a generalization of the proof for the exponential family (See e.g.~\cite{duchi}). We first consider the case where $\tilde{p}_{t|\theta} = [\tilde{p}_{t|\theta}(x)]_{x\in\mathcal{X}}$ is a finite-dimensional vector. The solution to this problem can be obtained by introducing Lagrange multipliers $\theta\in\mathbb{R}$, $\lambda\in\mathbb{R}$, and $\nu \geq 0$ to enforce the constraints
\begin{equation}
\begin{split}
\label{eq:exp_t_density_lag}
\tilde{p}_{t|\theta}(x)= \argmin_{\tilde{p}} \Big\{-H_t(\tilde{p}) - \theta \big(\int x\, \tilde{p}(x)\, \mathrm{d}\xi(x) - \mu\big)\\
+ \lambda\,\big(\int \tilde{p}(x)^{2-t}\, \mathrm{d}\xi(x) - 1\big) - \nu\, \tilde{p}(x)\Big\}\, ,
\end{split}
\end{equation}
where $H_t$ is Tsallis' entropy, defined in \eqref{eq:tsallis} (main file). Setting the functional derivative with respect to $p(x)$ to zero yields
\[
\log_t \tilde{p}(x) - \theta x + \lambda'\,\tilde{p}(x)^{1-t} - \nu = 0\,,
\]
with $\lambda' = (2-t)\,\lambda$. Expanding the definition of $\log_t$, we can rewrite
the equation as
\[
(1 + (1 - t)\, \lambda')\, \tilde{p}(x)^{1-t} = 1 + (1-t)\,\theta x + \nu'\, .
\]
By the KKT conditions, $\nu' = (1-t)\,\nu$ is zero iff $1 + (1-t)\,\theta x \geq 0$. We remark that for this to hold, we need $1+(1-t)\lambda'\ge0$. We suppose it holds and then we will check that it does indeed hold. Using the definition of $\exp_t$, the equation becomes 
\begin{eqnarray}
  \exp_t(\lambda')\,\tilde{p}(x)=\exp_t(\theta  x), \label{eqGT}
  \end{eqnarray}
which is thus the general form of an \acrotem. Denoting $\lambda'$ by $G_t(\theta)$ yields the form of
Eq.~\eqref{eq:exp_t_density_form}.  Next, we show that the solution holds for any event space $\mathcal{X}$. For $\psi_t \defeq z\log_t z -\log_{t-1} z$ (originally defined in~\citep{bitemp}), we let\begin{equation}
    \label{eq:bregman_t}
    D_{\psi_t}(u, v) = u \log_t u - u \log_t v - \log_{t-1} u + \log_{t-1} v 
  \end{equation}
denote the (scalar) Bregman divergence induced by $\psi_t$ and by extension $$D_{\psi_t}(\tilde{P}, \hat{\tilde{P}}) \defeq \int D_{\psi_t}(\tilde{p}(x), \hat{\tilde{p}}(x))\, \mathrm{d} \xi.$$ Consider any $\tilde{P} \in \tilde{\mathcal{P}}_{t|\ve{\hbar}}$ with unnormalized density $\tilde{p}(x)$. We have
\begin{align*}
    H_t(\tilde{P}) & = - \!\!\int_{\mathcal{X}} \psi_t(\tilde{p}(x))\, \mathrm{d}\xi
    \\&= - \!\!\int_{\mathcal{X}} \tilde{p}(x)\log_t \tilde{p}(x)\, \mathrm{d}\xi  + \!\!\int_{\mathcal{X}} \tilde{p}(x)\log_t \tilde{p}_{t|\theta}(x)\, \mathrm{d}\xi - \!\!\int_{\mathcal{X}} \tilde{p}(x)\log_t \tilde{p}_{t|\theta}(x)\, \mathrm{d}\xi \\
& = -D_{\psi_t}(\tilde{P}, \tilde{P}_{t|\theta}) - \theta \hbar + \left(\int_{\mathcal{X}} \tilde{p}(x)\tilde{p}_{t|\theta}(x)^{1-t} \mathrm{d}\xi\right)  G(\theta)\, .
\intertext{By adding and subtracting $G(\theta)$ and refactoring the terms, we have}
H_t(\tilde{P}) & = -D_{\psi_t}(\tilde{P}, \tilde{P}_{t|\theta}) (1 + (1 - t) G(\theta)) - (\theta \hbar - G(\theta))\\
& = -D_{\psi_t}(\tilde{P}, \tilde{P}_{t|\theta}) \exp_t(G_t(\theta))^{1-t} + H_t(\tilde{P}_{t|\theta})\, ,
\end{align*}
where we use the fact that $\E_{\tilde{P}}[\phi(x)] = \E_{\tilde{P}_{t|\theta}}[\phi(x)] = \hbar$ and $G_t(\theta) \geq - 1/(1-t)$ by the fact that \eqref{eqGT}, the denomination $\lambda' \defeq G_t(\theta)$ and the normalization constraint of the dual \acroted, we obtain
\begin{eqnarray}
G_t(\theta) & = & \log_t \left(\int \exp_t(\theta \phi(x))^{2-t} \mathrm{d}\xi \right)^{\frac{1}{2-t}} = (\log_t)^* \int (\exp_t)^* (\theta \phi(x))\mathrm{d}\xi,
\end{eqnarray}
which since $\log_t(z) \geq -1/(1-t)$, shows $G_t(\theta) \geq - 1/(1-t)$ and confirms $1+(1-t)\lambda'\ge0$.

To finish up, we check that
\begin{eqnarray}
  \frac{\partial G_t(\theta)}{\partial \theta} & = & \left( \frac{\int (\exp_t)^* (\theta \phi(x))\mathrm{d}\xi}{t^*} \right)^{-t^*} \cdot \int \phi(x) \exp_t (\theta \phi(x))\mathrm{d}\xi\\
                                               & = & \left( \frac{\int (\exp_t)^* (\theta \phi(x))\mathrm{d}\xi}{t^*} \right)^{-t^*} \cdot \exp_t G_t(\theta) \cdot \hbar\\
                                               & = & \left( \frac{\int (\exp_t)^* (\theta \phi(x))\mathrm{d}\xi}{t^*} \right)^{-t^*} \cdot \left( \frac{\int (\exp_t)^* (\theta \phi(x))\mathrm{d}\xi}{t^*} \right)^{t^*} \cdot \hbar\\
  & = & \hbar,
  \end{eqnarray}
as claimed.
  
\section{Proof of Lemma \ref{lemMf}}\label{sec-proof-lemMf}
We get the result using the $\log_t$ entropy with two different derivations,
\begin{eqnarray*}
\expect_{\tilde{p}_{t|\ve{\theta}}} \left[\log_t \tilde{p}_{t|\ve{\theta}}\right] & = & \int \tilde{p}_{t|\ve{\theta}}\, \frac{1}{1-t}\big(\tilde{p}_{t|\ve{\theta}}^{1-t} - 1\big)\, \mathrm{d}\xi = \frac{1}{1-t}\, (1 - \mathrm{M}_t(\ve{\theta}))\\
                                                                  & = & \int \tilde{p}_{t|\ve{\theta}}\, \big(\ve{\theta}^\top\ve{\phi} - \tilde{p}_{t|\hat{\ve{\theta}}}^{1-t}\,G_t(\ve{\theta})\big) \mathrm{d}\xi = \ve{\theta}^\top\ve{\hbar} - G_t(\ve{\theta}),
\end{eqnarray*}
and we identify the right-hand sides to get the statement of the Lemma. In the upmost derivation, we use the definition of $\mathrm{M}_t(\ve{\theta})$ and the fact that $\tilde{p}_{t|\ve{\theta}}^{2-t}$ sums to 1. In the bottommost derivation, we use the expression in \eqref{eq:exp_t_density_form} (main file) to identify the terms between the integrals and then simplify. We get $\mathrm{M}_t({\ve{\theta}}) = 1 + (1-t)(G_t(\ve{\theta}) - \ve{\theta}^\top\ve{\hbar})$. If $G_t$ is strictly convex differentiable, since by the relationship $\ve{\theta} = \nabla G_t^{-1}(\ve{\hbar})$ and convex duality, $G^\star_t(\ve{\hbar}) = \ve{\theta}^\top\ve{\hbar} - G_t(\ve{\theta})$,
  \begin{eqnarray}
    \mathrm{M}_t({\ve{\theta}}) = 1 + (1-t) (-G^\star_t(\ve{\hbar})) \quad (=\exp_t^{1-t} (-G^\star_t(\ve{\hbar}))).\label{eqMts}
  \end{eqnarray}
  \begin{remark}\label{remm}
    The non-negativity of the total mass $\mathrm{M}_t$ gives us a non-trivial lowerbound for $G_t$ and upperbound for $G^\star_t$:
  \begin{eqnarray}
    G_t(\ve{\theta}) & \geq & -\frac{1}{1-t} + \ve{\theta}^\top\ve{\hbar},\\
    G^\star_t(\ve{\hbar}) & \leq & \frac{1}{1-t},
  \end{eqnarray}
  both of which become vacuous when $t\rightarrow 1$.
  \end{remark}

\section{Proof of Theorem \ref{thm-ITIG}}\label{sec-proof-thm-ITIG}
Using the $t$-algebra and the definition of $F_t$ in \eqref{eq-gen-FD} (main
file), we first get an integral-free expression:
\begin{eqnarray*} 
F_{t}(\tilde{P}_{t|\hat{\ve{\theta}}}\|\tilde{P}_{t|\ve{\theta}})& = & \int f\left(\frac{\mathrm{d}\tilde{p}_{t|\hat{\ve{\theta}}}}{\mathrm{d}\xi} \oslash_t \frac{\mathrm{d}\tilde{p}_{t|{\ve{\theta}}}}{\mathrm{d}\xi} \right) \mathrm{d}\tilde{p}_{t|{\ve{\theta}}} \\
  & = & \int - \log_{t} \left( \exp_t(\hat{\ve{\theta}}^\top\ve{\phi} \ominus_t G_t(\hat{\ve{\theta}})) \oslash_t \exp_t({\ve{\theta}}^\top\ve{\phi} \ominus_t G_t({\ve{\theta}})) \right) \mathrm{d}\tilde{p}_{t|{\ve{\theta}}} \\
  & = & \int \left({\ve{\theta}}^\top\ve{\phi} \ominus_t G_t({\ve{\theta}}) - \hat{\ve{\theta}}^\top\ve{\phi} \ominus_t G_t(\hat{\ve{\theta}})\right)\mathrm{d}\tilde{p}_{t|{\ve{\theta}}} \\
  & = & \frac{{\ve{\theta}}^\top\ve{\mu} -  \mathrm{M}_t({\ve{\theta}})G_t({\ve{\theta}})}{1+(1-t)G_t({\ve{\theta}})}  - \frac{\hat{\ve{\theta}}^\top\ve{\mu} -  \mathrm{M}_t({\ve{\theta}})G_t(\hat{\ve{\theta}})}{1+(1-t)G_t(\hat{\ve{\theta}})},
\end{eqnarray*} 
and we then simplify the last expression using Lemma \ref{lemMf}:
\begin{eqnarray*}
  \lefteqn{\frac{{\ve{\theta}}^\top\ve{\mu} -  \mathrm{M}_t({\ve{\theta}})G_t({\ve{\theta}})}{1+(1-t)G_t({\ve{\theta}})} -  \frac{\hat{\ve{\theta}}^\top\ve{\mu} -  \mathrm{M}_t({\ve{\theta}})G_t(\hat{\ve{\theta}})}{1+(1-t)G_t(\hat{\ve{\theta}})}} \\
  & = & \frac{{\ve{\theta}}^\top\ve{\mu} -  (1 + (1-t)(G_t(\ve{\theta}) - \ve{\theta}^\top\ve{\mu})) G_t({\ve{\theta}})}{1+(1-t)G_t({\ve{\theta}})} -\frac{\hat{\ve{\theta}}^\top\ve{\mu} -  \mathrm{M}_t({\ve{\theta}})G_t(\hat{\ve{\theta}})}{1+(1-t)G_t(\hat{\ve{\theta}})} \\
  & = & {\ve{\theta}}^\top\ve{\mu}  - G_t({\ve{\theta}}) -\frac{\hat{\ve{\theta}}^\top\ve{\mu} -  \mathrm{M}_t({\ve{\theta}})G_t(\hat{\ve{\theta}})}{1+(1-t)G_t(\hat{\ve{\theta}})}\\
  & = & \frac{{\ve{\theta}}^\top\ve{\mu}  - G_t({\ve{\theta}}) -\hat{\ve{\theta}}^\top\ve{\mu} + \left((1-t)({\ve{\theta}}^\top\ve{\mu}  - G_t({\ve{\theta}})) + \mathrm{M}_t({\ve{\theta}})\right) G_t(\hat{\ve{\theta}})}{1+(1-t)G_t(\hat{\ve{\theta}})} \\
  & = & \frac{{\ve{\theta}}^\top\ve{\mu}  - G_t({\ve{\theta}}) -\hat{\ve{\theta}}^\top\ve{\mu} + G_t(\hat{\ve{\theta}})}{1+(1-t)G_t(\hat{\ve{\theta}})}  = \frac{G_t(\hat{\ve{\theta}}) - G_t({\ve{\theta}}) -(\hat{\ve{\theta}} - {\ve{\theta}})^\top\ve{\mu}  }{1+(1-t)G_t(\hat{\ve{\theta}})} \\
  & \defeq & B_{G_t}(\hat{\ve{\theta}} \| \ve{\theta}),
\end{eqnarray*}
which yields the statement of the Theorem.

\begin{remark}
We remark that $F_{t}(\tilde{P}_{t|\hat{\ve{\theta}}}\|\tilde{P}_{t|\ve{\theta}})$
is also equal to the Bregman divergence 
$D_{\psi_t}(\tilde{P}_{t|{\ve{\theta}}}\|\tilde{P}_{t|\hat{\ve{\theta}}})$, a connection also known to hold for exponential families' analysis where the KL divergence is both an $f$-divergence and a Bregman divergence.
  \end{remark}

\section{Proof of Lemma \ref{lem-left-popmin}}\label{sec-proof-lem-left-popmin}

Recall that $\mathsf{T}_i(\ve{\theta}) \defeq G_t(\ve{\theta}_i) + (\ve{\theta} -\ve{\theta}_i)^\top \nabla G_t(\ve{\theta}_i)$ the value at $\ve{\theta}$ of the tangent hyperplane to $G_t$ at $\ve{\theta}_i$. Denote for shot
\begin{eqnarray}
  N(\ve{\theta}) & \defeq & 1 + (1-t) \expect_i [\mathsf{T}_i(\ve{\theta})],\\
  D(\ve{\theta}) & \defeq & 1+(1-t) G_t(\ve{\theta}).
\end{eqnarray}
We then obtain the loss for the left population minimizer:
\begin{eqnarray}
L_{\mathrm{l}} (\ve{\theta}) & = & \frac{1}{1-t}\cdot \left(1 - \frac{N}{D}(\ve{\theta})\right),\label{defLossL}
\end{eqnarray}
which immediately yields:
\begin{lemma}\label{lemCP}
  $\ve{\theta}$ is a critical point of $L_{\mathrm{l}} (\ve{\theta})$ iff:
  \begin{eqnarray}
    N(\ve{\theta}) \cdot \nabla D(\ve{\theta}) & = & D(\ve{\theta}) \cdot \nabla N(\ve{\theta}).\label{eqCP}
  \end{eqnarray}
\end{lemma}
\begin{lemma}\label{lemNPOS}
Suppose $\exists i : \mathrm{M}_t({\ve{\theta}}_i) > 0$ and $G_t$ is strictly convex or strictly concave. Then any critical point of $L_{\mathrm{l}} (\ve{\theta})$ has $N(\ve{\theta}) \neq 0$.
\end{lemma}
\begin{proof}
  Suppose otherwise. Note that unless we are in the degenerate case where all $\ve{\theta}_i$ are equal, $D(\ve{\theta}) > N(\ve{\theta})$ from the strict convexity of $G_t$\footnote{If strictly concave, $D(\ve{\theta}) < N(\ve{\theta})$, which yields to the same result.} (we recall that $N(\ve{\theta})$ is the expected value of all tangent hyperplanes at all $\ve{\theta}_i$s, at $\ve{\theta}$, which thus sits strictly below the function). So Lemma \ref{lemCP} implies $\nabla N(\ve{\theta}) = \ve{0}$, which, after developing, is in fact
  \begin{eqnarray}
\expect_i \nabla G_t(\ve{\theta}_i) & = & \ve{0},
  \end{eqnarray}
  In addition to being a critical point, the condition $N(\ve{\theta}) \neq 0$ yields $\expect_i [\mathsf{T}_i(\ve{\theta})] = -1/(1-t)$. Using the definition of $\mathsf{T}_i$ and simplifying with $\expect_i \nabla G_t(\ve{\theta}_i) = \ve{0}$ then reveals
\begin{eqnarray*}
\expect_i [G_t(\ve{\theta}_i) + (\ve{\theta} -\ve{\theta}_i)^\top \nabla G_t(\ve{\theta}_i)] = \expect_i [G_t(\ve{\theta}_i) -\ve{\theta}_i^\top \nabla G_t(\ve{\theta}_i)] = -\frac{1}{1-t},
  \end{eqnarray*}
  which, using Lemma \ref{lemMf} (main file) reveals that the population for which $\ve{\theta}$ is a minimizer necessarily has
  \begin{eqnarray*}
    \expect_i [\mathrm{M}_t({\ve{\theta}}_i)] & = & 1 + (1-t)\expect_i [G_t(\ve{\theta}_i) -\ve{\theta}_i^\top \nabla G_t(\ve{\theta}_i)]\\
    & = & 1 -1 = 0.
  \end{eqnarray*}
  Since $\mathrm{M}_t$ is non-negative, this leads to a contradiction with the assumption of the Lemma.
  \end{proof}
The question is then whether such a critical point can be a population minimizer, and even more, if it is unique. Answering the first question comes from the Hessian of the loss.
\begin{lemma}\label{lemHessian}
  Removing the argument $\ve{\theta}$ for readability, we have:
  \begin{eqnarray}
\mathrm{H}L_{\mathrm{l}} & = & \frac{1-t}{D^2}\cdot \left(\nabla \expect_i [\mathsf{T}_i] \nabla G_t^\top+\nabla G_t \nabla \expect_i [\mathsf{T}_i]^\top\right)- \frac{2(1-t)N}{D^3}\cdot \nabla G_t \nabla G_t^\top +\frac{N}{D^2}\cdot \mathrm{H} G_t .\label{eqHessian}
  \end{eqnarray}
\end{lemma}
\begin{proof}
Denote for short $\updelta_{ij} \defeq [\nabla G_t(\ve{\theta}_i)]_j$ and $\nabla_j \defeq [\nabla G_t(\ve{\theta})]_j$. We check that we have $(\partial / \partial \theta'_k) \expect_i [\mathsf{T}_i] = \updelta_{ik}$, so we have
\begin{eqnarray*}
  \lefteqn{\frac{\partial}{\partial \theta'_j} [\nabla_{\ve{\theta}} L_{\mathrm{l}}]_k}\\
  & = & \frac{\partial}{\partial \theta'_j} \left(\frac{(1 + (1-t) \expect_i [\mathsf{T}_i])\nabla_k }{(1+(1-t) G_t(\ve{\theta}))^2}- \frac{\expect_i \updelta_{ik}}{1+(1-t) G_t(\ve{\theta})} \right)\\
  & = & \frac{\left\{\begin{array}{c}
                       \left((1-t)\updelta_{ij} \nabla_k + (1 + (1-t) \expect_i [\mathsf{T}_i])\nabla_{kj}\right) (1+(1-t) G_t(\ve{\theta}))^2 \\
                       - 2(1-t) \left((1 + (1-t) \expect_i [\mathsf{T}_i])\nabla_k\right)(1+(1-t) G_t(\ve{\theta})) \nabla_j\\
                     \end{array}\right\}}{(1+(1-t) G_t(\ve{\theta}))^4}  + \frac{(1-t)\expect_i \updelta_{ik} \nabla_j}{(1+(1-t) G_t(\ve{\theta}))^2}\\
  & = & \frac{(1-t)\updelta_{ij} \nabla_k + N \nabla_{kj}}{D^2} - \frac{2(1-t) N\nabla_k\nabla_j}{D^3} + \frac{(1-t)\expect_i \updelta_{ik} \nabla_j}{D^2}\\
  & = & \frac{(1-t)}{D^2}\cdot(\updelta_{ij}\nabla_k + \nabla_j \updelta_{ik} )  - \frac{2(1-t) N}{D^3}\cdot \nabla_k\nabla_j +  \frac{N}{D^2}\cdot  \nabla_{kj};
\end{eqnarray*}
noting our convention yields $\nabla_{kj} = [\mathrm{H} G_t]_{jk}$, we get the statement of the Lemma.
\end{proof}
The next one introduces the convexity of $G_t$ to elicit the nature of the critical points.:
\begin{lemma}
At any critical point of $L_{\mathrm{l}}$, the convexity of $L_{\mathrm{l}}$ is the same as the convexity of $G_t$ iff $N \geq 0$ (and it is opposed, meaning convex$\leftrightarrow$concave, otherwise).
  \end{lemma}
\begin{proof}
  Using Lemma \ref{lemCP}, \eqref{eqHessian} simplifies to
  \begin{eqnarray*}
    \mathrm{H}L_{\mathrm{l}} & = & \frac{(1-t)N}{D^3}\cdot \left(\nabla G_t \nabla G_t^\top+\nabla G_t \nabla G_t^\top\right)- \frac{2(1-t)N}{D^3}\cdot \nabla G_t \nabla G_t^\top +\frac{N}{D^2}\cdot \mathrm{H} G_t \\
    & = & \frac{N}{D^2}\cdot \mathrm{H} G_t ,
  \end{eqnarray*}
  yielding the statement of the Lemma.
\end{proof}
Hence, if $N \geq 0$, all critical points are population minimizers. In the next Lemma, we show a condition for unicity.
\begin{lemma}
Suppose $G_t$ is strictly convex. Any optimum of $L_{\mathrm{l}}$ is unique.
\end{lemma}
\begin{proof}
  Let us consider any two such minimizers $\ve{\theta}', \ve{\theta}''$. We thus have simultaneously from Lemma \ref{lemCP}:
\begin{eqnarray}
  N(\ve{\theta}') \cdot \nabla D(\ve{\theta}') = D(\ve{\theta}') \cdot \nabla N(\ve{\theta}') = (1-t) D(\ve{\theta}') \cdot \expect_i \nabla G_t(\ve{\theta}_i) \label{cond2},\\
  N(\ve{\theta}'') \cdot \nabla D(\ve{\theta}'') = D(\ve{\theta}'') \cdot \nabla N(\ve{\theta}'') = (1-t) D(\ve{\theta}'') \cdot \expect_i \nabla G_t(\ve{\theta}_i)\label{cond3},\\
 \frac{\frac{D(\ve{\theta}')-1}{1-t} - \frac{N(\ve{\theta}')-1}{1-t}}{D(\ve{\theta}')} = L_{\mathrm{l}} (\ve{\theta}') = L_{\mathrm{l}} (\ve{\theta}'') = \frac{\frac{D(\ve{\theta}'')-1}{1-t} - \frac{N(\ve{\theta}'')-1}{1-t}}{D(\ve{\theta}'')} \label{cond4}.
\end{eqnarray}
We note \eqref{cond4} is equivalent, after simplification, to
\begin{eqnarray}
  \frac{N(\ve{\theta}')}{D(\ve{\theta}')} & = & \frac{N(\ve{\theta}'')}{D(\ve{\theta}'')}.\label{simplcond4}
\end{eqnarray}
Also, \eqref{cond2} and \eqref{cond3} bring:
\begin{eqnarray*}
  \frac{N(\ve{\theta}')}{D(\ve{\theta}')} \cdot \left. \frac{\partial D(\ve{\theta})}{\partial \theta_i} \right|_{\ve{\theta} = \ve{\theta}'} & = &  \frac{N(\ve{\theta}'')}{D(\ve{\theta}'')} \cdot \left. \frac{\partial D(\ve{\theta})}{\partial \theta_i} \right|_{\ve{\theta} = \ve{\theta}''}, \forall i \in [d],
\end{eqnarray*}
and thus simplifies with \eqref{simplcond4} to $\left. \frac{\partial D(\ve{\theta})}{\partial \theta_i} \right|_{\ve{\theta} = \ve{\theta}'} = \left. \frac{\partial D(\ve{\theta})}{\partial \theta_i} \right|_{\ve{\theta} = \ve{\theta}''}, \forall i \in [d]$, or in gradient form after using the definition of $D$ and simplifying,
\begin{eqnarray*}
\nabla G_t(\ve{\theta}') & = & \nabla G_t(\ve{\theta}''),
\end{eqnarray*}
but since $G_t$ is strictly convex, $\nabla G_t$ is bijective and this implies $\ve{\theta}' = \ve{\theta}''$, and completes the proof of the Lemma.
\end{proof}
Folding together all Lemmata, we get that if $G_t$ is strictly convex, $\ve{\theta}$ is the unique left population minimizer if $\nabla L_{\mathrm{l}} (\ve{\theta}) = \ve{0}$ and $N(\ve{\theta}) > 0$. \eqref{eqCP} can be reformulated as:
\begin{eqnarray}
\nabla G_t(\ve{\theta}) & = & \frac{D(\ve{\theta})}{N(\ve{\theta})} \cdot \expect_i \nabla G_t(\ve{\theta}_i) \label{eqBreg1},
\end{eqnarray}
Now, define function $\alpha$:
\begin{eqnarray}
  \alpha(\ve{\theta}) & = & \frac{D(\ve{\theta})}{N(\ve{\theta})}.
\end{eqnarray}
We note from \eqref{eqBreg1} that $\alpha_* = \alpha(\ve{\theta}_{\mathrm{l}})$ and we also note from \eqref{defLossL} that we also have
\begin{eqnarray}
L_{\mathrm{l}} (\ve{\theta}) & = & \frac{1}{1-t}\cdot \left(1 - \frac{1}{\alpha(\ve{\theta})}\right),
\end{eqnarray}
so we conclude, if $t\leq 1$,
\begin{eqnarray}
\alpha_* & \leq & \min_i \alpha(\ve{\theta}_i),\label{balphastar}
  \end{eqnarray}
which provides a convenient upperbound which, in addition to the fact that $\alpha_* \geq 1$, provides a convenient initialisation interval for a line search of $\alpha_*$.

\section{Proof of Lemma \ref{lem-rob-right}}\label{sec-proof-lem-rob-right}
We recall the right population minimizer:
\begin{eqnarray}
\ve{\theta}^{\mathrm{old}}_{\mathrm{r}} & = & \expect_i\left[\frac{1}{1+(1-t)G_t(\ve{\theta}_i)} \cdot \ve{\theta}_i\right].
\end{eqnarray}
If we add a new point $\ve{\theta}_*$ with a weight $\epsilon$ and downweight the other points' weights proportionally, the new right population minimizer is
\begin{eqnarray}
\ve{\theta}^{\mathrm{new}}_{\mathrm{r}} & = & (1-\epsilon)\cdot \ve{\theta}^{\mathrm{old}}_{\mathrm{r}} + \epsilon \cdot \frac{1}{1+(1-t)G_t(\ve{\theta}_*)} \cdot \ve{\theta}_*,
\end{eqnarray}
and so
\begin{eqnarray}
\ve{\theta}^{\mathrm{new}}_{\mathrm{r}} - \ve{\theta}^{\mathrm{old}}_{\mathrm{r}}& = & \epsilon \cdot \underbrace{\left(\frac{1}{1+(1-t)G_t(\ve{\theta}_*)} \cdot \ve{\theta}_* - \ve{\theta}^{\mathrm{old}}_{\mathrm{r}} \right)}_{\defeq \ve{z}(\ve{\theta}_*)},
\end{eqnarray}
and so if $G_t(\ve{\theta}_i) = \Omega(\|\ve{\theta}_*\|)$, $\|\ve{z}(\ve{\theta}_*)\| \leq (1/(1+(1-t)G_t(\ve{\theta}_*)))\cdot \|\ve{\theta}_*\| + \|\ve{\theta}^{\mathrm{old}}_{\mathrm{r}}\| = O(1)$ and the right population minimizer is robust.

\section{Proof of Lemma \ref{lem-rob-left}}\label{sec-proof-lem-rob-left}

Denote for short
\begin{eqnarray}
  \ve{\theta}^{\mathrm{old}}_{\mathrm{l}} & \defeq & \arg\min_{\ve{\theta}} L(\ve{\theta}) \defeq \frac{\expect_i[D_{G_t}(\ve{\theta}\|\ve{\theta}_i)]}{1+(1-t) G_t(\ve{\theta})}\label{defTOLD}\\
  \ve{\theta}^{\mathrm{new}}_{\mathrm{l}} & \defeq & \arg\min_{\ve{\theta}} L^\epsilon(\ve{\theta}) \defeq \frac{(1-\epsilon) \expect_i[D_{G_t}(\ve{\theta}\|\ve{\theta}_i)] + \epsilon D_{G_t}(\ve{\theta}\|\ve{\theta}_*)}{1+(1-t) G_t(\ve{\theta})}.
\end{eqnarray}
Also, we let $\nabla_\epsilon \defeq (1-\epsilon) \cdot \nabla G_t(\ve{\theta}^{\mathrm{old}}_{\mathrm{l}}) + \epsilon \cdot \nabla G_t(\ve{\theta}_*)$ and $\expect(\ve{\theta})\defeq  \expect_i[\mathsf{T}_i(\ve{\theta})]  $, $\expect_{\epsilon}(\ve{\theta})\defeq (1-\epsilon) \expect(\ve{\theta}) + \epsilon \mathsf{T}_*(\ve{\theta})$, where by extension $\mathsf{T}_*(\ve{\theta}) \defeq G_t(\ve{\theta}_*) + (\ve{\theta} -\ve{\theta}_*)^\top \nabla G_t(\ve{\theta}_*)$. 
We also use the following Taylor expansion:
\begin{itemize}
\item [(A)] $\nabla G_t(\ve{\theta}^{\mathrm{new}}_{\mathrm{l}}) -\nabla G_t(\ve{\theta}^{\mathrm{old}}_{\mathrm{l}}) = \mathrm{H}_t (\ve{\theta}^{\mathrm{new}}_{\mathrm{l}} - \ve{\theta}^{\mathrm{old}}_{\mathrm{l}}) \defeq \epsilon \mathrm{H}_t \ve{z}_t(\ve{\theta}_*)$, where $\mathrm{H}_t$ is a value of the Hessian of $G_t$.
  \end{itemize}
  Finally, we note
  \begin{eqnarray}
  \nabla L^\epsilon(\ve{\theta}) & = & \frac{1}{(1+(1-t)G_t(\ve{\theta}))^2}\cdot \left((1+(1-t)\expect_{\epsilon}(\ve{\theta}))\cdot \nabla G_t(\ve{\theta}) - (1+(1-t)G_t(\ve{\theta}))\cdot \nabla_\epsilon\right).
\end{eqnarray}
Using the definition of $\ve{\theta}^{\mathrm{new}}_{\mathrm{l}}$ and (A), we get to:
\begin{eqnarray}
  \ve{0} & = & (1+(1-t)G_t(\ve{\theta}^{\mathrm{new}}_{\mathrm{l}}))^2\cdot \nabla L^\epsilon(\ve{\theta}^{\mathrm{new}}_{\mathrm{l}})\nonumber\\
         & = & \left((1+(1-t)\expect_{\epsilon}(\ve{\theta}^{\mathrm{new}}_{\mathrm{l}}))\cdot \nabla G_t(\ve{\theta}^{\mathrm{new}}_{\mathrm{l}}) - (1+(1-t)G_t(\ve{\theta}^{\mathrm{new}}_{\mathrm{l}}))\cdot \nabla_\epsilon\right)\nonumber\\
         & = & \left((1+(1-t)\expect_{\epsilon}(\ve{\theta}^{\mathrm{new}}_{\mathrm{l}}))\cdot (\nabla G_t(\ve{\theta}^{\mathrm{old}}_{\mathrm{l}}) + \epsilon \mathrm{H}_t \ve{z}_t(\ve{\theta}_*)) - (1+(1-t)G_t(\ve{\theta}^{\mathrm{new}}_{\mathrm{l}}))\cdot \nabla_\epsilon\right).\label{eqrob1}
\end{eqnarray}
We get the relationship satisfied by the influence function:
\begin{eqnarray}
  \lefteqn{\ve{z}_t(\ve{\theta}_*)}\nonumber\\
  & = & \frac{1}{\epsilon} \cdot \mathrm{H}_t^{-1}\left(\frac{1+(1-t)G_t(\ve{\theta}^{\mathrm{new}}_{\mathrm{l}})}{1+(1-t)\expect_{\epsilon}(\ve{\theta}^{\mathrm{new}}_{\mathrm{l}})}\cdot \nabla_\epsilon - \nabla G_t(\ve{\theta}^{\mathrm{old}}_{\mathrm{l}})\right)\nonumber\\
  & = & \mathrm{H}_t^{-1}\left(\frac{1+(1-t)G_t(\ve{\theta}^{\mathrm{new}}_{\mathrm{l}})}{1+(1-t)\expect_{\epsilon}(\ve{\theta}^{\mathrm{new}}_{\mathrm{l}})}\cdot \nabla G_t(\ve{\theta}_*) + \left(\frac{1-\epsilon}{\epsilon} \cdot \frac{1+(1-t)G_t(\ve{\theta}^{\mathrm{new}}_{\mathrm{l}})}{1+(1-t)\expect_{\epsilon}(\ve{\theta}^{\mathrm{new}}_{\mathrm{l}})} - 1\right)\cdot \nabla G_t(\ve{\theta}^{\mathrm{old}}_{\mathrm{l}})\right)\nonumber\\
& = & \mathrm{H}_t^{-1}\left(\alpha(\ve{\theta}^{\mathrm{new}}_{\mathrm{l}})\cdot \nabla G_t(\ve{\theta}_*) + \left(\frac{1-\epsilon}{\epsilon} \cdot \alpha(\ve{\theta}^{\mathrm{new}}_{\mathrm{l}}) - 1\right)\cdot \nabla G_t(\ve{\theta}^{\mathrm{old}}_{\mathrm{l}})\right).
\end{eqnarray}
We know from \eqref{balphastar} that $\alpha(\ve{\theta}^{\mathrm{new}}_{\mathrm{l}})$ cannot diverge as a function of the outlier $\ve{\theta}_*$, so we end up with
\begin{eqnarray}
  \ve{z}_t(\ve{\theta}_*) & = & Q \cdot \mathrm{H}_t^{-1} \nabla G_t(\ve{\theta}_*) + \mathrm{H}_t^{-1} \ve{v},\label{eqZt}
\end{eqnarray}
where $Q \ll \infty, \|\ve{v}\| \ll \infty$. If we compute the $f$-mean for $G_t$ and its influence function, then we get this time:
\begin{eqnarray}
  \ve{z}_1(\ve{\theta}_*) & = & \frac{1}{\epsilon} \cdot \mathrm{H}_1^{-1}  \nabla_\epsilon\\
                          & = & \mathrm{H}_1^{-1}  \left(\nabla G_t(\ve{\theta}_*) + \frac{1-\epsilon}{\epsilon} \cdot \expect_i[\nabla G_t(\ve{\theta}_i)]\right)\\
  & = & \mathrm{H}_1^{-1} \nabla G_t(\ve{\theta}_*) + \mathrm{H}_1^{-1} \ve{v}_1,\label{eqZ1}
\end{eqnarray}
where $\|\ve{v}_1\| \ll \infty$. We see that $\ve{z}_t$ has bounded norm iff $\ve{z}_1$ does so, which proves the statement of the Lemma.

\begin{remark}
  The strong convexity argument is here just to handle the influence of the Hessian via its minimal eigenvalue in \eqref{eqZt} and \eqref{eqZ1}. We could add (realistic) assumptions on the training sample's domain to replace the strong convexity argument by strict convexity.
  \end{remark}
  
\section{Proof of Theorem \ref{th-conti}}\label{sec-proof-th-conti}

We proceed in three steps.

\noindent \textbf{Step 1}: $\forall \ve{\theta}, \forall \varepsilon > 0, \exists t<1 : G_t(\ve{\theta}) \geq G_1(\ve{\theta}) - \varepsilon$. We rely on the inequalities\footnote{The proofs, at the end of the proof of Theorem \ref{th-conti}, elicit $a, b$.}:
  \begin{eqnarray}
  \forall t\in [0,1], \forall z\geq 0, \log(z) & \leq & (\log_t)^*(z),\label{approx01}\\
  \exists a,b \in \mathbb{R} \mbox{ s.t. } \forall t\in [0,1], \forall z \in \mathbb{R}, \underbrace{\left(1-(1-t)g_{a,b}(z)\right)_+}_{=\exp_{t}^{1-t}(-g(z))}\exp(z) & \leq & (\exp_t)^*(z),\label{approx02}
  \end{eqnarray}
  where $g_{a,b}(z) \defeq a z^2 - b z + 1$. From \eqref{approx01} and \eqref{approx02}, we get the inequalities in:
  \begin{eqnarray}
    G_t(\ve{\theta}) & = & (\log_t)^* \int (\exp_t)^* \left(\ve{\theta}^\top \ve{\phi} \right)\mathrm{d}\xi\nonumber\\
    & \geq & \log \int (\exp_t)^* \left(\ve{\theta}^\top \ve{\phi} \right)\mathrm{d}\xi \nonumber\\
    & \geq & \log \int \exp_{t}^{1-t}\left( -a (\ve{\theta}^\top \ve{\phi})^2 + b (\ve{\theta}^\top \ve{\phi}) - 1\right) \exp\left(\ve{\theta}^\top \ve{\phi} \right)\mathrm{d}\xi\nonumber\\
    & & = \log \int \left[ t + (1-t) \left(-a (\ve{\theta}^\top \ve{\phi})^2 + b(\ve{\theta}^\top \ve{\phi}) \right)\right]_+ \exp\left(\ve{\theta}^\top \ve{\phi} \right)\mathrm{d}\xi\nonumber\\
    & = & G_1(\ve{\theta}) + \log \int \left[ t + (1-t) \left(-a (\ve{\theta}^\top \ve{\phi})^2 + b(\ve{\theta}^\top \ve{\phi}) \right)\right]_+ \exp\left(\ve{\theta}^\top \ve{\phi}-G_1(\ve{\theta}) \right)\mathrm{d}\xi\nonumber\\
    & \defeq & G_1(\ve{\theta}) + \log Q_t(\ve{\theta}), \label{leq1}
  \end{eqnarray}
  with $Q_t(\ve{\theta}) \defeq \expect_1 \left[ f_t(\ve{\theta}^\top \ve{\phi})\right]$ ($\expect_1$ indicating the expectation for $t=1$, \textit{i.e.} the exponential family) and
  \begin{eqnarray}
f_t(z) & \defeq & \left[ t + (1-t) \cdot z \left(b - a z\right) \right]_+.
  \end{eqnarray}
  Since $f_t$ does not take negative values, for any $0\leq \delta < 1$, if we let $\mathcal{X}_\delta \defeq \{\ve{x} \in \mathcal{X} : f_t(\ve{\theta}^\top \ve{\phi}(\ve{x})) \geq \delta\}$, then we have, for $\mu_1$ the probability measure associated to $t=1$,
  \begin{eqnarray}
Q_t(\ve{\theta}) & \geq & \delta \cdot \mu_1 (\mathcal{X}_\delta).
  \end{eqnarray}
  Also, for any $a, b, z \in \mathbb{R}$, $f_1(z) = 1$ and $f_t(z)$ is continuous in $t$ and $z$ so
  \begin{eqnarray*}
    \forall z_* > 0, \forall 0\leq \delta < 1, \exists t < 1 \mbox{ s.t. } f_t([-z_*,z_*]) \subseteq [\delta, +\infty),
  \end{eqnarray*}
and for any applicable $t_\delta<1$, any $t \in [t_\delta, 1)$ is also valid. This implies that $\forall 0\leq \delta < 1$, we can always find $t_\delta < 1$ close enough to $1$ such that $\mu_1 (\mathcal{X}_\delta) \geq \delta$ by picking $z_*$ large enough. So, for any $0\leq \delta < 1$, we can find $t_\delta$ such that $Q_t(\ve{\theta}) \geq \delta \cdot \delta = \delta^2$, and if we choose $\delta \defeq \exp(-\varepsilon/2)$, then $\log Q_t(\ve{\theta}) \geq -\varepsilon$ and considering \eqref{leq1}, we obtain:
  \begin{eqnarray}
\forall \ve{\theta}, \forall \varepsilon > 0, \exists t<1 : G_t(\ve{\theta}) \geq G_1(\ve{\theta}) - \varepsilon,
  \end{eqnarray}
  \textit{i.e.} we have completed the proof of \textbf{Step 1}.

\noindent \textbf{Step 2}: $\forall \ve{\theta}, \forall \varepsilon > 0, \exists t<1 : G_t(\ve{\theta}) \leq G_1(\ve{\theta}) + \varepsilon$. We rely on the inequalities\footnote{The proofs, at the end of the proof of Theorem \ref{th-conti}, elicit functions $u(t), v(t)$.}:
  \begin{eqnarray}
    \forall t\in [0,1], \forall z\geq 0, (\log_t)^*(z) & \leq & \log(z) - u(t) z + v(t) (1+ \log^2 z+  (z-1)^2),\label{approx03}\\
  (\exp_t)^*(z) & \leq & \exp(z) , \forall t\leq 1, \forall z \in \mathbb{R}, \label{approx04}
  \end{eqnarray}
  where $u(t), v(t)$ are two continuous functions of $t$ satisfying $u(1) = v(1) = 0$. $\log_t$ being non-decreasing, we get with \eqref{approx04} the first inequality of
  \begin{eqnarray}
    G_t(\ve{\theta}) & \defeq &  (\log_t)^* \int (\exp_t)^* \left(\ve{\theta}^\top \ve{\phi} \right)\mathrm{d}\xi\nonumber\\
    & \leq &  (\log_t)^* \int \exp \left(\ve{\theta}^\top \ve{\phi} \right)\mathrm{d}\xi\nonumber\\
                     & \leq & \log \int \exp \left(\ve{\theta}^\top \ve{\phi} \right)\mathrm{d}\xi  - u(t) \cdot \int \exp \left(\ve{\theta}^\top \ve{\phi} \right)\mathrm{d}\xi \nonumber\\
    & & + v(t) \cdot \left(1 + \left(\log \int \exp \left(\ve{\theta}^\top \ve{\phi} \right)\mathrm{d}\xi\right)^2 + \left(\int \exp \left(\ve{\theta}^\top \ve{\phi} \right)\mathrm{d}\xi -1\right)^2\right)\nonumber\\
    & = & G_1(\ve{\theta}) \underbrace{- u(t) \cdot \exp(G_1(\ve{\theta})) + v(t) \cdot \left(1 + G^2_1(\ve{\theta}) + \left(\exp (G_1(\ve{\theta})) - 1\right)^2\right)}_{\defeq h_t(G_1(\ve{\theta}))}.
  \end{eqnarray}
We have $h_1 = 0$ and $h_t$ is continuous in $t$, so $\forall z \in \mathbb{R}, \forall \varepsilon > 0, \exists t<1$ close enough to $1$ such that $h_{t}(z) \leq \varepsilon$, implying, for $z \defeq G_1(\ve{\theta})$, $G_t(\ve{\theta}) \leq G_1(\ve{\theta}) + \varepsilon$. This completes the proof of \textbf{Step 2}.

  We then check that we can simultaneously get \textbf{Step 1} and \textbf{Step 2} as both depend on choosing $t<1$ close enough to $1$. What remains is then:

  \noindent \textbf{Step 3}: we show \eqref{approx01}, \eqref{approx02}, \eqref{approx03}, \eqref{approx04}.

  \noindent $\hookrightarrow$ We first prove \eqref{approx01} and define
  \begin{eqnarray*}
\Delta(z) & \defeq & (\log_t)^*(z) - \log(z).
  \end{eqnarray*}
  We have
  \begin{eqnarray}
\Delta'(z) & = & \left((2-t)z\right)^{-\frac{1}{2-t}} - \frac{1}{z} .
  \end{eqnarray}
  $\Delta'$ zeroes for $z_* = (2-t)^{1/(1-t)}$ (which is the global minimum for $\Delta$), for which $(\log_t)^*(z_*) = 1$. Since $\log(1+z) \leq z, \forall z> 0$, we get by picking $z=1-t$ and reorganizing $\log z_* \leq 1$, and thus $\Delta(z_*) \geq 0$, and since $z_*$ is the global minimum of $\Delta$, yields \eqref{approx01}.

    \noindent $\hookrightarrow$ Since $(\log_t)^*$ and $(\exp_t)^*$ are inverses of each other, \eqref{approx04} follows from \eqref{approx01}.

    \noindent $\hookrightarrow$ We now prove \eqref{approx02}. Equivalently, we show that for some $a,b \in \mathbb{R}$ we have
  \begin{eqnarray}
P_{a,b}(z) \defeq t + (1-t)z(b-az) & \leq & \exp(-z) \exp_t(z), \forall t \in [0,1], \forall z \in \mathbb{R} \label{propPab}
  \end{eqnarray}
(we note the result trivially holds for $t=1$ so we focus on $t\in [0,1)$).  If
  \begin{eqnarray}
a,b & > & 0 \label{choiceabcd1}
    \end{eqnarray}
    then $P_{a,b}(z)$  is a downwards facing parabola with its maximum in $z_* \defeq b/(2a) > 0$ and it always has two roots:
    \begin{eqnarray}
z_\pm & \defeq & z_*\cdot\left(1 \pm \sqrt{1+\frac{4at}{b^2(1-t)}}\right).
    \end{eqnarray}
    We now want to choose $a,b$ so as to constrain
    \begin{eqnarray}
z_-, z_+ \in \mathbb{I}, \forall t \in [0,1), \mathbb{I} \defeq \sqrt{\frac{1}{1-t}}\cdot [-1,1].\label{constzstar}
    \end{eqnarray}
    \begin{enumerate}
    \item Case of $z_-$. Since $z_- < 0$, we just need $z_- \geq -\sqrt{1/(1-t)}$, which after reorganising becomes:
      \begin{eqnarray}
\sqrt{1+\frac{4at}{b^2(1-t)}} & \leq & 1 + \sqrt{a} \cdot \sqrt{\frac{4a}{b^2(1-t)}}.
      \end{eqnarray}
      To get this, it is sufficient we want the same inequality with a $t$ factor in the rightmost square root of the RHS (since $t\leq 1$). Making the change of variable $Z \defeq 4at/(b^2(1-t))$, which ranges through $\mathbb{R}_+$, we thus want $\sqrt{1+Z} \leq 1 + \sqrt{aZ}$, which indeed holds over $\mathbb{R}_+$ if
      \begin{eqnarray}
a & \geq & 1. \label{choiceabcd2}
      \end{eqnarray}
    \item Case of $z_+$. Since $z_+ > 0$, we just need $z_+ \leq \sqrt{1/(1-t)}$, which after reorganising becomes:
\begin{eqnarray}
        a & \geq & t + b\sqrt{1-t}.
\end{eqnarray}
The RHS takes its max for $b = 2\sqrt{1-t}$, for which it equals $2-t$. Hence, to get $z_+ \leq \sqrt{1/(1-t)}$, we just need
\begin{eqnarray}
a & \geq & 2. \label{choiceabcd3}
      \end{eqnarray}
    \end{enumerate}
    Hence, if $a\geq 2$, then \eqref{constzstar} holds. Given that is holds and since $\exp_t$ is an increasing function of $t$, to get \eqref{propPab}, it is enough that we prove that for some $a>2, b>0$,
    \begin{eqnarray}
P_{a,b}(z) & \leq & \exp_t(-z) \exp_t(z) = \underbrace{\left(1-(1-t)^2z^2\right)^{\frac{1}{1-t}}}_{\defeq Q(z)}, \forall z \in \mathbb{I}, \forall t \in [0,1) \label{propPab2}
    \end{eqnarray}
    We then note
    \begin{eqnarray}
Q'(z) = -2(1-t)\cdot zQ^t(z) & ; & Q''(z) = -2(1-t) \cdot (1-(1-t^2)z^2) Q^{2t-1}(z).
    \end{eqnarray}
    A Taylor expansion in $z=0$ then gives $Q(z) \sim_0 1 - (1-t)z^2 \defeq R(z)$. Noting $R'(z) = -2(1-t)\cdot z$ and since $Q(z) \leq 1$, we obtain $0\geq Q'(z) \geq R'(z)$ for $z\in \mathbb{I}_+$ and so $R(z) \leq Q(z), \forall z\in \mathbb{I}_+$. Since both functions are even, we thus get
    \begin{eqnarray}
R(z) & \leq & Q(z), \forall z\in \mathbb{I}.
    \end{eqnarray}
    To get \eqref{propPab2}, we thus just need $P_{a,b}(z) \leq R(z), \forall z \in \mathbb{I}$. Since $t\leq 1$, this inequality has the convenient $t$-free simplification
    \begin{eqnarray}
(a-1) z^2 - bz + 1 & \geq & 0, \forall z \in \mathbb{I}.
    \end{eqnarray}
    This parabola facing upwards has no root (and is thus non negative) if
    \begin{eqnarray}
a & \geq & 1 + \frac{b^2}{4}.
      \end{eqnarray}
      To summarise, we get \eqref{propPab} (and so \eqref{approx02}) for any choice $a, b$ satisfying:
      \begin{eqnarray}
b > 0 & ; & a \geq \max\left\{2, 1 + \frac{b^2}{4}\right\}.
      \end{eqnarray}
        \noindent $\hookrightarrow$ We finish by showing \eqref{approx03}. We want to show
      \begin{eqnarray}
    \forall t\in [0,1], \forall z\geq 0, (\log_t)^*(z) & \leq & \underbrace{\log(z) - u(t) z + v(t) (1+ \log^2 z+  (z-1)^2)}_{\defeq h_t(z)}\label{blogzstar}
      \end{eqnarray}
      For some $u(t) \geq 0, v(t) \geq 0$ both defined on $[0,1]$, continuous and with limit 0 in $t=1^-$. 
      We fix
      \begin{eqnarray}
u(t) \defeq 1 - t_*^{t_*}, & ; & v(t) \defeq (\log_t)^*(1) + u(t),
      \end{eqnarray}
      and we check they trivially satisfy those properties in addition to being strictly decreasing over $[0,1]$ and satisfying $u([0,1]) = [0,1-1/\sqrt{2}], v([0,1]) = [0,1/\sqrt{2}]$. We also check that \eqref{blogzstar} trivially holds for $t=1$ so we prove the result for $t\in [0,1)$. We have
      \begin{eqnarray}
(\log_t)^*(1) = h_t(1) ;\quad  {(\log_t)^*}'(1) = h_t'(1),\label{plogh}
      \end{eqnarray}
so functions $(\log_t)^*, h_t$ are tangent at $z=1$, $\forall t \in [0,1]$.  We note
       \begin{eqnarray*}
         h'_t(z) & = & \frac{1-(u(t)+2v(t))z + 2v(t)(z^2+\log z)}{z} .
       \end{eqnarray*}
       Given \eqref{plogh}, if we can show
       \begin{eqnarray}
h_t'(z) & \leq & {(\log_t)^*}'(z), \forall z \leq 1, \forall t \in [0,1)\label{eq111}
       \end{eqnarray}
       then, since all related functions are continuous, we obtain
       \begin{eqnarray}
 {(\log_t)^*}(z) & \leq & h_t(z), \forall z \leq 1, \forall t \in [0,1).
       \end{eqnarray}
       \eqref{eq111} is the same as
       \begin{eqnarray}
\underbrace{1-(u(t)+2v(t))z + 2v(t)(z^2+\log z)}_{\defeq i_t(z)} & \leq & \underbrace{(1-u(t))z^{1-t_*}}_{j_t(z)}, \forall z \leq 1, \forall t \in [0,1). \label{eq112}
       \end{eqnarray}
       Since $i_t(1) = j_t(1)$, \eqref{eq112} is guaranteed if $i'_t(z) \geq j'_t(z), \forall z \leq 1, \forall t \in [0,1)$. This condition can be formulated as (for any $c$):
       \begin{eqnarray}
4v(t)z + \frac{c}{z} + \frac{2v(t) - c}{z} & \geq & u(t) + 2v(t) + \frac{(1-t_*)(1-u(t))}{z^{t_*}}, \forall z \leq 1. \label{eq113}
       \end{eqnarray}
       Now pick $c \defeq (u(t) + 2v(t))^2 / (16v(t))$. We can check that
       \begin{eqnarray}
 c \leq v(t), \forall t \in [0,1] & ; & 4v(t)z + \frac{c}{z} \geq u(t) + 2v(t), \forall z \geq 0, \forall t \in [0,1),
         \end{eqnarray}
         so we get
         \begin{eqnarray}
4v(t)z + \frac{c}{z} + \frac{2v(t) - c}{z} & \geq & u(t) + 2v(t) + \frac{v(t)}{z}, \forall z \leq 1, \forall t \in [0,1). 
       \end{eqnarray}
       To get \eqref{eq113}, it is thus enough, since $t_* \in [1/2, 1], u(t) \in [0,1], z\leq 1$, that we show $v(t) \geq 1-t_*$, which after reordering, yields equivalently $(\log_t)^*(1) \geq t_*^{t_*} - t_*$. Using $t_* \defeq 1/(2-t)$ and multiplying both sides by $2-t$ yields in compact form the requirement
       \begin{eqnarray}
(2-t) (\log_t)^*(1) & \geq & (1-t) (\log_t)^*(1),
       \end{eqnarray}
       which, since $(\log_t)^*(1) > 0$ for $t\in [0,1)$, indeed holds. In summary, we have shown:
       \begin{eqnarray}
 {(\log_t)^*}(z) & \leq & h_t(z), \forall z \leq 1, \forall t \in [0,1).
       \end{eqnarray}
       There remains to cover the cases $z>1$ and it is sufficient to change the polarity of \eqref{eq111}, \eqref{eq112} and thus show
       \begin{eqnarray}
i_t(z) & \geq & j_t(z), \forall z \geq 1, \forall t \in [0,1). \label{eq114}
       \end{eqnarray}
       We now restrict the interval to check for $t$. We remark that
\begin{eqnarray*}
         h_t''(z) & = & \frac{v(t)}{z^2} \cdot \left(2z^2-\log z - w(t)\right), \quad w(t) \defeq \frac{1-v(t)}{v(t)};
\end{eqnarray*}
We note $w(t) \in [\sqrt{2}-1, +\infty)$. Since $v(t) \geq 0$, for all $t$s such that $(2z^2-\log z - w(t))([1,+\infty))$ does not contain 0, $h_t$ is convex for $z\geq 1$. Since $(\log_t)^*$ is concave, we shall get our result. What is the set of such $t$s ? Function $z \mapsto 2z^2-\log z - w(t)$ is strictly increasing for $z\geq 1$. Thus, we seek $t$ such that $w(t) \leq 2$, or equivalently $v(t) \geq 1/3$: for any $t$ such that $v(t) \geq 1/3$, $h_t$ is convex over $[1,+\infty)$ and our result \eqref{blogzstar} holds. We thus refine \eqref{eq114} by checking
\begin{eqnarray}
i_t(z) & \geq & j_t(z), \forall z \geq 1, \forall t \in [0,1) : v(t) \leq 1/3. \label{eq114}
\end{eqnarray}
Since $z^{1-t_*} \leq z$ and $\log z\geq 0$ for $z \geq 1$, \eqref{eq114} is implied by showing $1-(u(t)+2v(t))z + 2v(t) z^2 \geq (1-u(t))z$ (we recall $v(t) \geq 0$), which provides us with the degree-2 polynomial condition
\begin{eqnarray}
1 - (1+2v(t))z + 2v(t) z^2 & \geq & 0,
\end{eqnarray}
and this needs to be checked for $z\geq 1$, $t \in [0,1) : v(t) \leq 1/3$. We compute the roots
\begin{eqnarray}
z_{\pm} & \defeq & \frac{1+2v(t) \pm |1-2v(t)|}{2},
\end{eqnarray}
and check that the largest root, under the condition $v(t) \leq 1/3$, is $z_+ = (1/2)(1+2v(t) +1-2v(t)) = 1$. In other words, \eqref{eq114} holds and we have completed the proof of \eqref{approx03}, and thus the proof of Theorem \ref{th-conti}.

%% file: content/appendix-experiments.tex
\section{Voronoi diagrams}\label{sec-voronoi}

Figures \ref{fig:voronoi-left-full} and \ref{fig:voronoi-right-full} present more detailed Voronoi diagrams for the same setting as described in the main file.

  \setlength\tabcolsep{0pt}
  \newcommand{\vorosizefull}{0.098}
  
  \begin{figure*}
  \centering
  \begin{tabular}{c|cccccccccc}\hline\hline
  \rotatebox{90}{$t=0$} & \includegraphics[trim=0bp 0bp 0bp 0bp,clip,width=\vorosizefull\textwidth]{Voronoi/clustering_Experiment_-1_Iter_5JUST_VORONOI_LEFT_T_EXP_Iter_0_0.0_Learn_Div_T_EXP_LocCenter_LEFT_CENTER_T_0.0} & \includegraphics[trim=0bp 0bp 0bp 0bp,clip,width=\vorosizefull\textwidth]{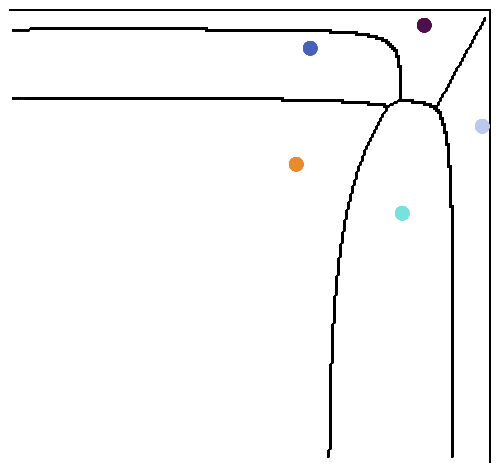} & \includegraphics[trim=0bp 0bp 0bp 0bp,clip,width=\vorosizefull\textwidth]{Voronoi/clustering_Experiment_-1_Iter_5JUST_VORONOI_LEFT_T_EXP_Iter_2_0.0_Learn_Div_T_EXP_LocCenter_LEFT_CENTER_T_0.0} & \includegraphics[trim=0bp 0bp 0bp 0bp,clip,width=\vorosizefull\textwidth]{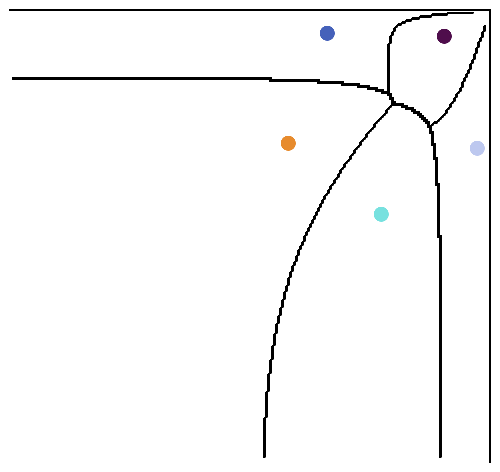} & \includegraphics[trim=0bp 0bp 0bp 0bp,clip,width=\vorosizefull\textwidth]{Voronoi/clustering_Experiment_-1_Iter_5JUST_VORONOI_LEFT_T_EXP_Iter_4_0.0_Learn_Div_T_EXP_LocCenter_LEFT_CENTER_T_0.0} & \includegraphics[trim=0bp 0bp 0bp 0bp,clip,width=\vorosizefull\textwidth]{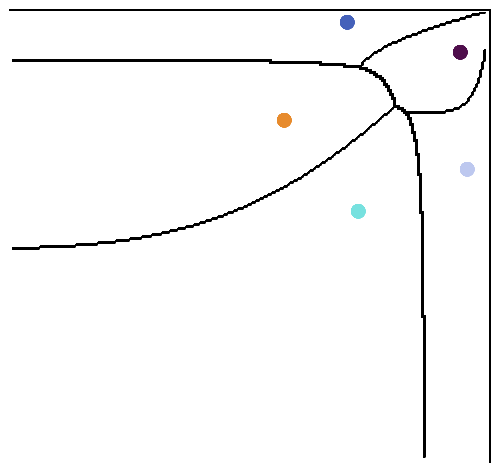} & \includegraphics[trim=0bp 0bp 0bp 0bp,clip,width=\vorosizefull\textwidth]{Voronoi/clustering_Experiment_-1_Iter_5JUST_VORONOI_LEFT_T_EXP_Iter_6_0.0_Learn_Div_T_EXP_LocCenter_LEFT_CENTER_T_0.0} & \includegraphics[trim=0bp 0bp 0bp 0bp,clip,width=\vorosizefull\textwidth]{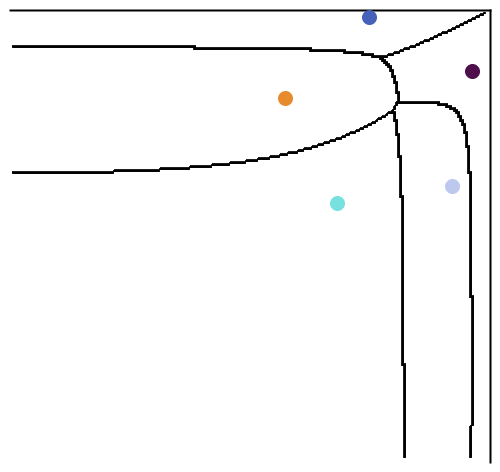} & \includegraphics[trim=0bp 0bp 0bp 0bp,clip,width=\vorosizefull\textwidth]{Voronoi/clustering_Experiment_-1_Iter_5JUST_VORONOI_LEFT_T_EXP_Iter_8_0.0_Learn_Div_T_EXP_LocCenter_LEFT_CENTER_T_0.0} & \includegraphics[trim=0bp 0bp 0bp 0bp,clip,width=\vorosizefull\textwidth]{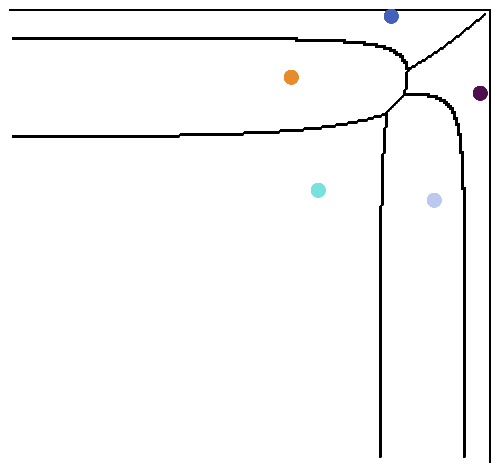} \\   \rotatebox{90}{$t=0.5$} & \includegraphics[trim=0bp 0bp 0bp 0bp,clip,width=\vorosizefull\textwidth]{Voronoi/clustering_Experiment_-1_Iter_5JUST_VORONOI_LEFT_T_EXP_Iter_0_0.5_Learn_Div_T_EXP_LocCenter_LEFT_CENTER_T_0.5} & \includegraphics[trim=0bp 0bp 0bp 0bp,clip,width=\vorosizefull\textwidth]{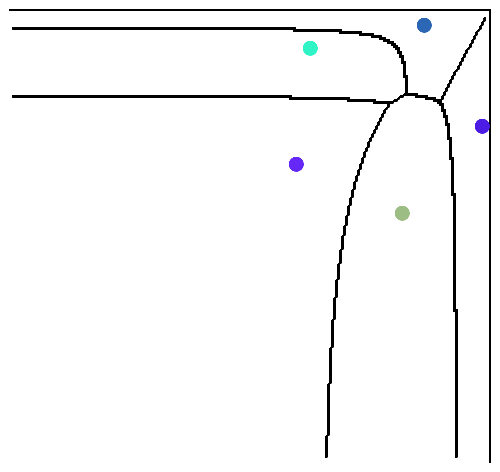} & \includegraphics[trim=0bp 0bp 0bp 0bp,clip,width=\vorosizefull\textwidth]{Voronoi/clustering_Experiment_-1_Iter_5JUST_VORONOI_LEFT_T_EXP_Iter_2_0.5_Learn_Div_T_EXP_LocCenter_LEFT_CENTER_T_0.5} & \includegraphics[trim=0bp 0bp 0bp 0bp,clip,width=\vorosizefull\textwidth]{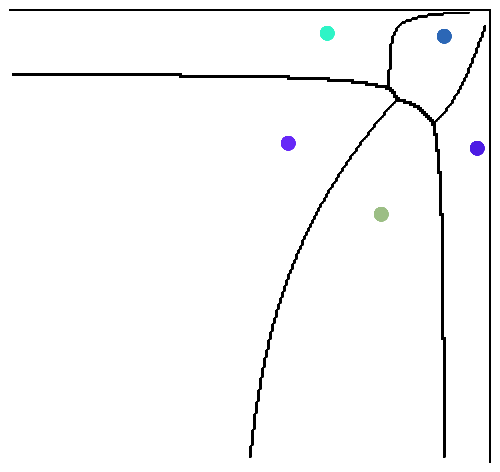} & \includegraphics[trim=0bp 0bp 0bp 0bp,clip,width=\vorosizefull\textwidth]{Voronoi/clustering_Experiment_-1_Iter_5JUST_VORONOI_LEFT_T_EXP_Iter_4_0.5_Learn_Div_T_EXP_LocCenter_LEFT_CENTER_T_0.5} & \includegraphics[trim=0bp 0bp 0bp 0bp,clip,width=\vorosizefull\textwidth]{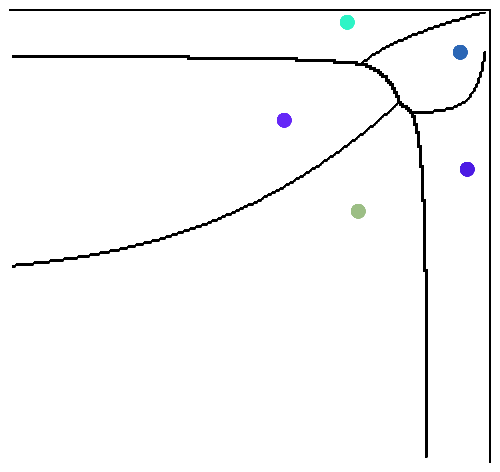} & \includegraphics[trim=0bp 0bp 0bp 0bp,clip,width=\vorosizefull\textwidth]{Voronoi/clustering_Experiment_-1_Iter_5JUST_VORONOI_LEFT_T_EXP_Iter_6_0.5_Learn_Div_T_EXP_LocCenter_LEFT_CENTER_T_0.5} & \includegraphics[trim=0bp 0bp 0bp 0bp,clip,width=\vorosizefull\textwidth]{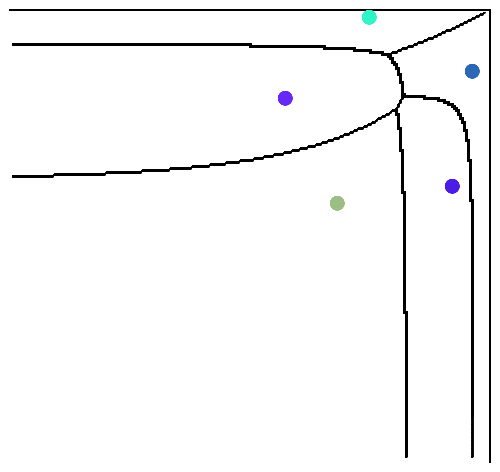} & \includegraphics[trim=0bp 0bp 0bp 0bp,clip,width=\vorosizefull\textwidth]{Voronoi/clustering_Experiment_-1_Iter_5JUST_VORONOI_LEFT_T_EXP_Iter_8_0.5_Learn_Div_T_EXP_LocCenter_LEFT_CENTER_T_0.5} & \includegraphics[trim=0bp 0bp 0bp 0bp,clip,width=\vorosizefull\textwidth]{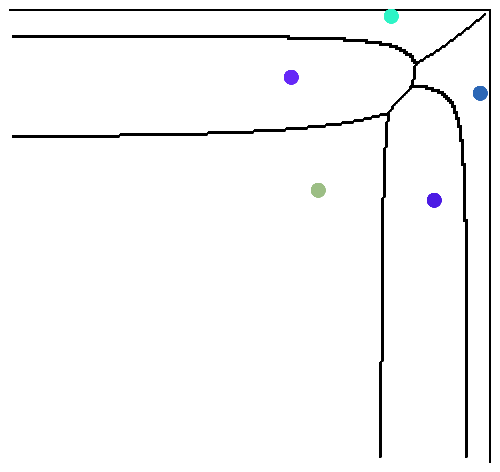} \\ \rotatebox{90}{$t=1.0$} & \includegraphics[trim=0bp 0bp 0bp 0bp,clip,width=\vorosizefull\textwidth]{Voronoi/clustering_Experiment_-1_Iter_5JUST_VORONOI_LEFT_T_EXP_Iter_0_1.0_Learn_Div_T_EXP_LocCenter_LEFT_CENTER_T_1.0} & \includegraphics[trim=0bp 0bp 0bp 0bp,clip,width=\vorosizefull\textwidth]{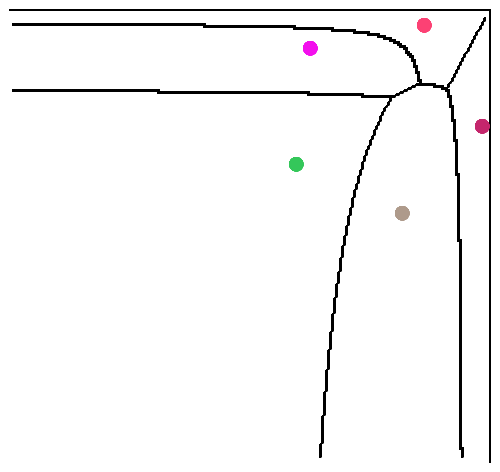} & \includegraphics[trim=0bp 0bp 0bp 0bp,clip,width=\vorosizefull\textwidth]{Voronoi/clustering_Experiment_-1_Iter_5JUST_VORONOI_LEFT_T_EXP_Iter_2_1.0_Learn_Div_T_EXP_LocCenter_LEFT_CENTER_T_1.0} & \includegraphics[trim=0bp 0bp 0bp 0bp,clip,width=\vorosizefull\textwidth]{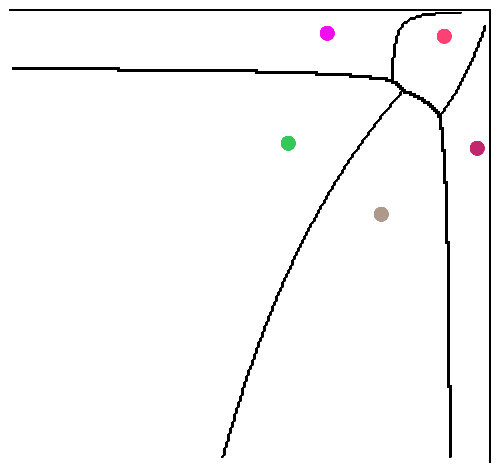} & \includegraphics[trim=0bp 0bp 0bp 0bp,clip,width=\vorosizefull\textwidth]{Voronoi/clustering_Experiment_-1_Iter_5JUST_VORONOI_LEFT_T_EXP_Iter_4_1.0_Learn_Div_T_EXP_LocCenter_LEFT_CENTER_T_1.0} & \includegraphics[trim=0bp 0bp 0bp 0bp,clip,width=\vorosizefull\textwidth]{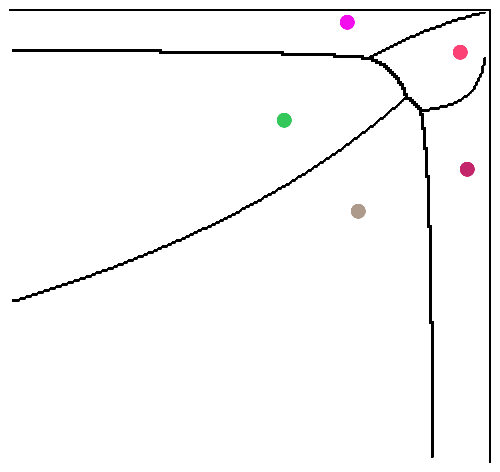} & \includegraphics[trim=0bp 0bp 0bp 0bp,clip,width=\vorosizefull\textwidth]{Voronoi/clustering_Experiment_-1_Iter_5JUST_VORONOI_LEFT_T_EXP_Iter_6_1.0_Learn_Div_T_EXP_LocCenter_LEFT_CENTER_T_1.0} & \includegraphics[trim=0bp 0bp 0bp 0bp,clip,width=\vorosizefull\textwidth]{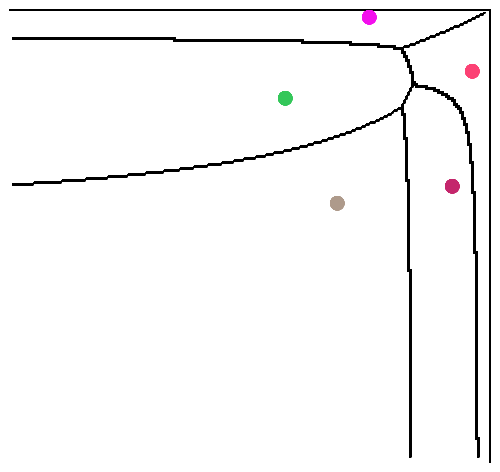} & \includegraphics[trim=0bp 0bp 0bp 0bp,clip,width=\vorosizefull\textwidth]{Voronoi/clustering_Experiment_-1_Iter_5JUST_VORONOI_LEFT_T_EXP_Iter_8_1.0_Learn_Div_T_EXP_LocCenter_LEFT_CENTER_T_1.0} & \includegraphics[trim=0bp 0bp 0bp 0bp,clip,width=\vorosizefull\textwidth]{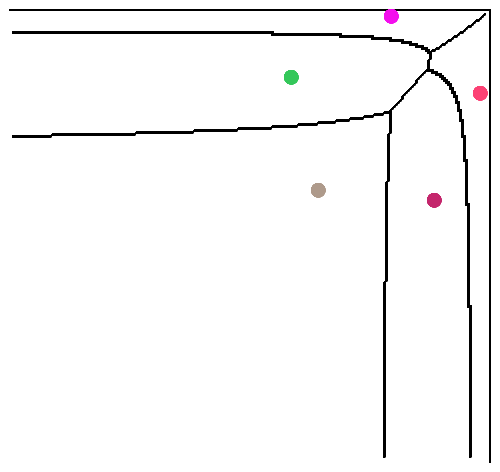} \\ \hline
    \end{tabular}
\caption{Voronoi diagrams associated to the left population minimizer of the 1D $t$-exponential (domain = $\mathbb{R}_{-*}^2$) of the vertices of a rotating regular pentagon, for $t \in \{0, 0.5, 1\}$.}
    \label{fig:voronoi-left-full}
  \end{figure*}

  \begin{figure*}
  \centering
  \begin{tabular}{c|cccccccccc}\hline\hline
  \rotatebox{90}{$t=0$} & \includegraphics[trim=0bp 0bp 0bp 0bp,clip,width=\vorosizefull\textwidth]{Voronoi/clustering_Experiment_-1_Iter_5JUST_VORONOI_RIGHT_T_EXP_Iter_0_0.0_Learn_Div_T_EXP_LocCenter_RIGHT_CENTER_T_0.0} & \includegraphics[trim=0bp 0bp 0bp 0bp,clip,width=\vorosizefull\textwidth]{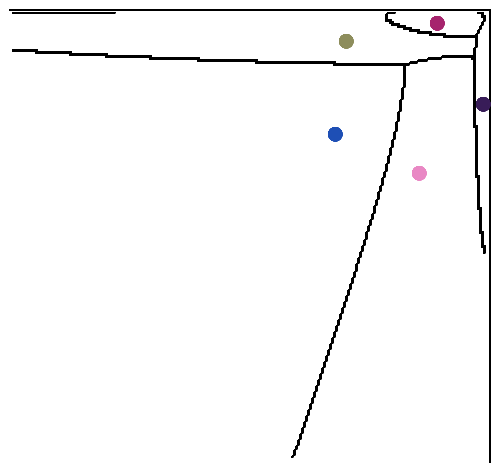} & \includegraphics[trim=0bp 0bp 0bp 0bp,clip,width=\vorosizefull\textwidth]{Voronoi/clustering_Experiment_-1_Iter_5JUST_VORONOI_RIGHT_T_EXP_Iter_2_0.0_Learn_Div_T_EXP_LocCenter_RIGHT_CENTER_T_0.0} & \includegraphics[trim=0bp 0bp 0bp 0bp,clip,width=\vorosizefull\textwidth]{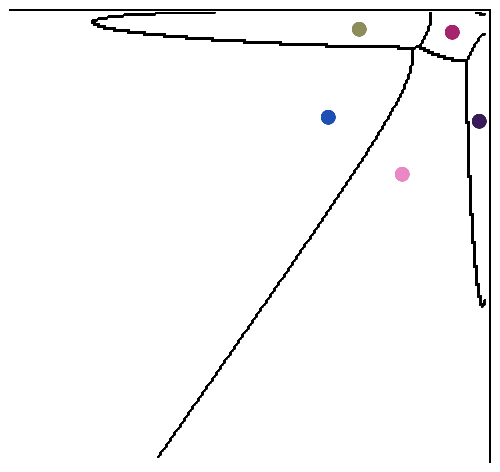} & \includegraphics[trim=0bp 0bp 0bp 0bp,clip,width=\vorosizefull\textwidth]{Voronoi/clustering_Experiment_-1_Iter_5JUST_VORONOI_RIGHT_T_EXP_Iter_4_0.0_Learn_Div_T_EXP_LocCenter_RIGHT_CENTER_T_0.0} & \includegraphics[trim=0bp 0bp 0bp 0bp,clip,width=\vorosizefull\textwidth]{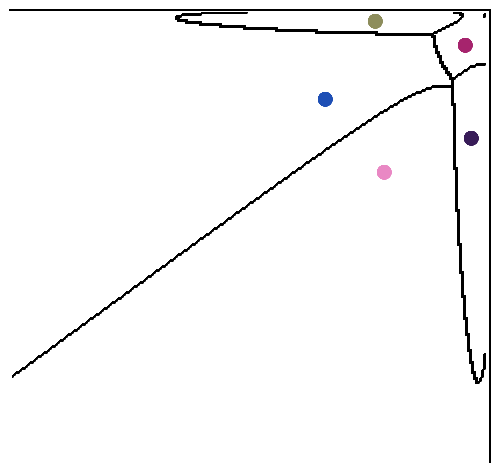} & \includegraphics[trim=0bp 0bp 0bp 0bp,clip,width=\vorosizefull\textwidth]{Voronoi/clustering_Experiment_-1_Iter_5JUST_VORONOI_RIGHT_T_EXP_Iter_6_0.0_Learn_Div_T_EXP_LocCenter_RIGHT_CENTER_T_0.0} & \includegraphics[trim=0bp 0bp 0bp 0bp,clip,width=\vorosizefull\textwidth]{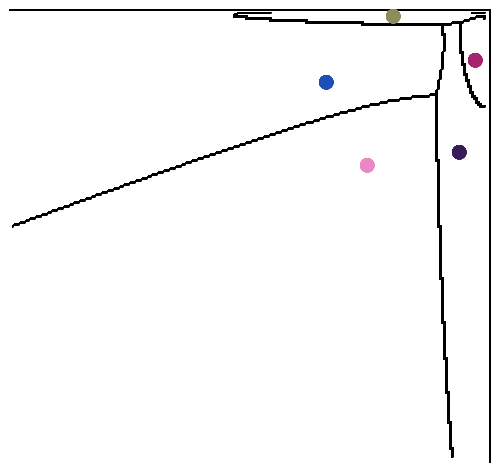} & \includegraphics[trim=0bp 0bp 0bp 0bp,clip,width=\vorosizefull\textwidth]{Voronoi/clustering_Experiment_-1_Iter_5JUST_VORONOI_RIGHT_T_EXP_Iter_8_0.0_Learn_Div_T_EXP_LocCenter_RIGHT_CENTER_T_0.0} & \includegraphics[trim=0bp 0bp 0bp 0bp,clip,width=\vorosizefull\textwidth]{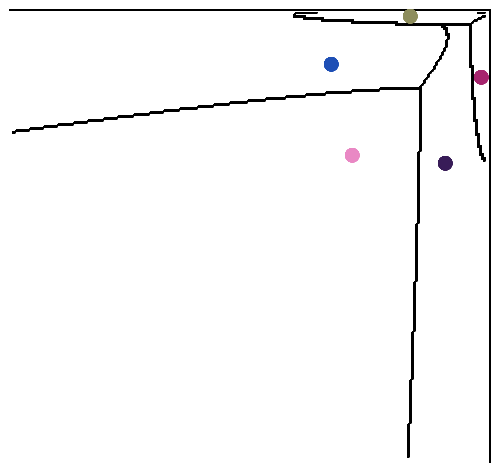} \\   \rotatebox{90}{$t=0.5$} & \includegraphics[trim=0bp 0bp 0bp 0bp,clip,width=\vorosizefull\textwidth]{Voronoi/clustering_Experiment_-1_Iter_5JUST_VORONOI_RIGHT_T_EXP_Iter_0_0.5_Learn_Div_T_EXP_LocCenter_RIGHT_CENTER_T_0.5} & \includegraphics[trim=0bp 0bp 0bp 0bp,clip,width=\vorosizefull\textwidth]{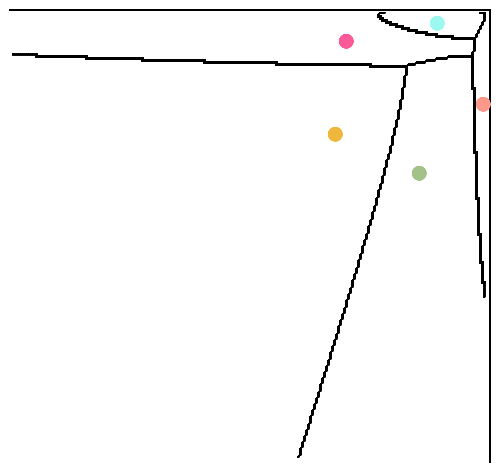} & \includegraphics[trim=0bp 0bp 0bp 0bp,clip,width=\vorosizefull\textwidth]{Voronoi/clustering_Experiment_-1_Iter_5JUST_VORONOI_RIGHT_T_EXP_Iter_2_0.5_Learn_Div_T_EXP_LocCenter_RIGHT_CENTER_T_0.5} & \includegraphics[trim=0bp 0bp 0bp 0bp,clip,width=\vorosizefull\textwidth]{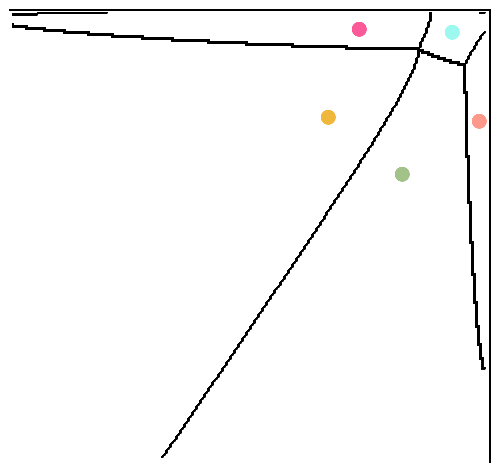} & \includegraphics[trim=0bp 0bp 0bp 0bp,clip,width=\vorosizefull\textwidth]{Voronoi/clustering_Experiment_-1_Iter_5JUST_VORONOI_RIGHT_T_EXP_Iter_4_0.5_Learn_Div_T_EXP_LocCenter_RIGHT_CENTER_T_0.5} & \includegraphics[trim=0bp 0bp 0bp 0bp,clip,width=\vorosizefull\textwidth]{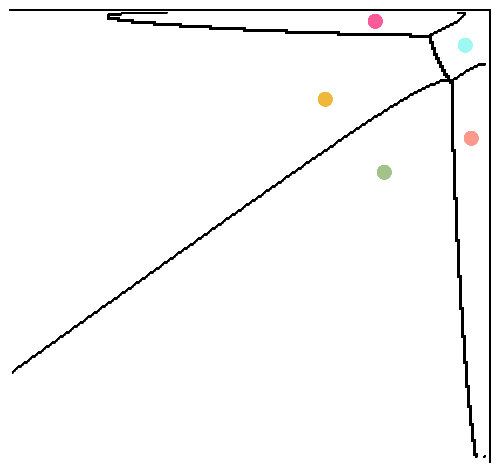} & \includegraphics[trim=0bp 0bp 0bp 0bp,clip,width=\vorosizefull\textwidth]{Voronoi/clustering_Experiment_-1_Iter_5JUST_VORONOI_RIGHT_T_EXP_Iter_6_0.5_Learn_Div_T_EXP_LocCenter_RIGHT_CENTER_T_0.5} & \includegraphics[trim=0bp 0bp 0bp 0bp,clip,width=\vorosizefull\textwidth]{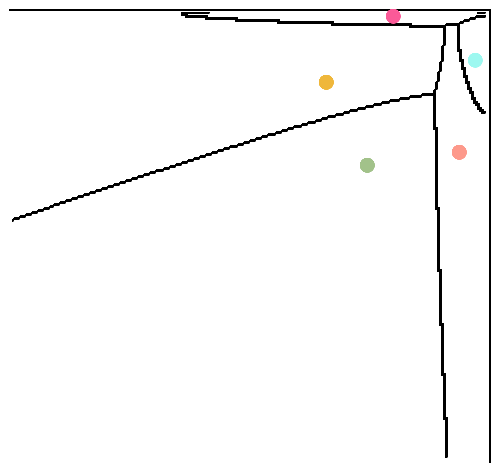} & \includegraphics[trim=0bp 0bp 0bp 0bp,clip,width=\vorosizefull\textwidth]{Voronoi/clustering_Experiment_-1_Iter_5JUST_VORONOI_RIGHT_T_EXP_Iter_8_0.5_Learn_Div_T_EXP_LocCenter_RIGHT_CENTER_T_0.5} & \includegraphics[trim=0bp 0bp 0bp 0bp,clip,width=\vorosizefull\textwidth]{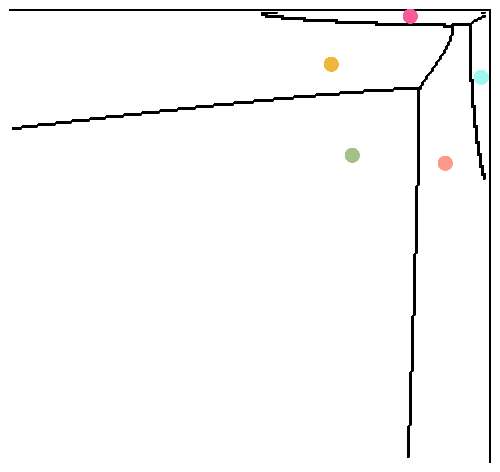} \\ \rotatebox{90}{$t=1.0$} & \includegraphics[trim=0bp 0bp 0bp 0bp,clip,width=\vorosizefull\textwidth]{Voronoi/clustering_Experiment_-1_Iter_5JUST_VORONOI_RIGHT_T_EXP_Iter_0_1.0_Learn_Div_T_EXP_LocCenter_RIGHT_CENTER_T_1.0} & \includegraphics[trim=0bp 0bp 0bp 0bp,clip,width=\vorosizefull\textwidth]{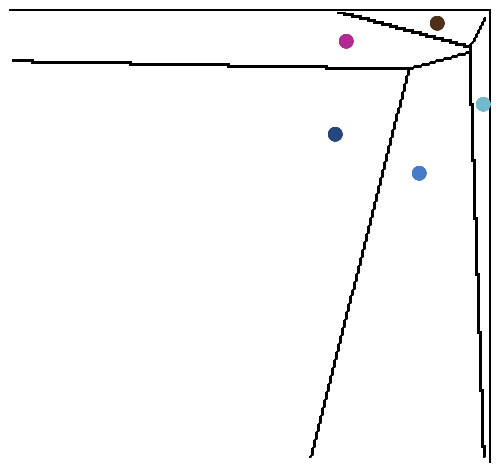} & \includegraphics[trim=0bp 0bp 0bp 0bp,clip,width=\vorosizefull\textwidth]{Voronoi/clustering_Experiment_-1_Iter_5JUST_VORONOI_RIGHT_T_EXP_Iter_2_1.0_Learn_Div_T_EXP_LocCenter_RIGHT_CENTER_T_1.0} & \includegraphics[trim=0bp 0bp 0bp 0bp,clip,width=\vorosizefull\textwidth]{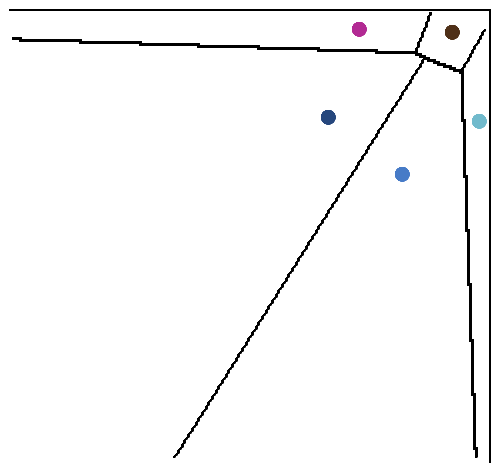} & \includegraphics[trim=0bp 0bp 0bp 0bp,clip,width=\vorosizefull\textwidth]{Voronoi/clustering_Experiment_-1_Iter_5JUST_VORONOI_RIGHT_T_EXP_Iter_4_1.0_Learn_Div_T_EXP_LocCenter_RIGHT_CENTER_T_1.0} & \includegraphics[trim=0bp 0bp 0bp 0bp,clip,width=\vorosizefull\textwidth]{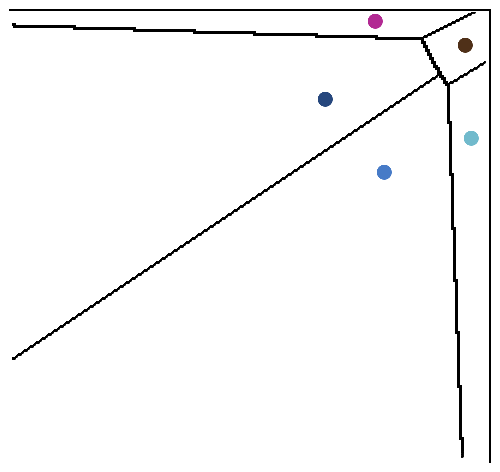} & \includegraphics[trim=0bp 0bp 0bp 0bp,clip,width=\vorosizefull\textwidth]{Voronoi/clustering_Experiment_-1_Iter_5JUST_VORONOI_RIGHT_T_EXP_Iter_6_1.0_Learn_Div_T_EXP_LocCenter_RIGHT_CENTER_T_1.0} & \includegraphics[trim=0bp 0bp 0bp 0bp,clip,width=\vorosizefull\textwidth]{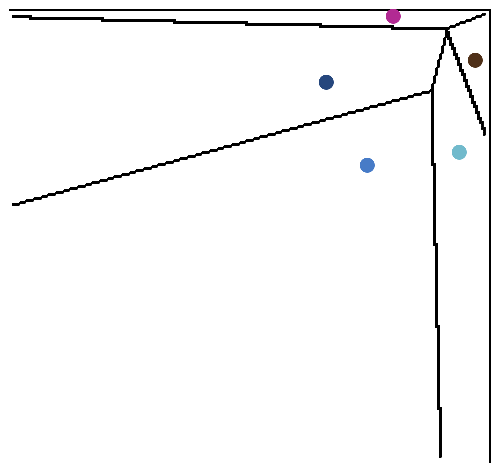} & \includegraphics[trim=0bp 0bp 0bp 0bp,clip,width=\vorosizefull\textwidth]{Voronoi/clustering_Experiment_-1_Iter_5JUST_VORONOI_RIGHT_T_EXP_Iter_8_1.0_Learn_Div_T_EXP_LocCenter_RIGHT_CENTER_T_1.0} & \includegraphics[trim=0bp 0bp 0bp 0bp,clip,width=\vorosizefull\textwidth]{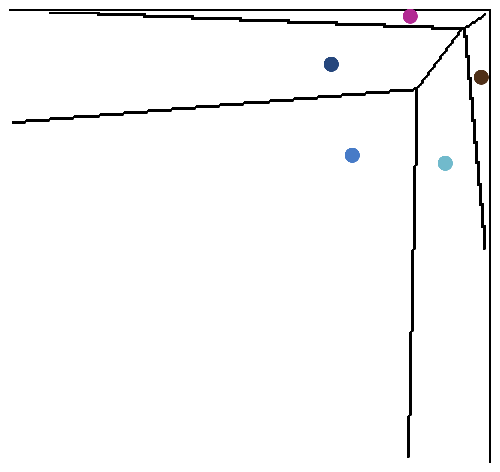} \\ \hline
    \end{tabular}
\caption{Voronoi diagrams associated to the right population minimizer of the 1D $t$-exponential (domain = $\mathbb{R}_{-*}^2$) of the vertices of a rotating regular pentagon, for $t \in \{0, 0.5, 1\}$.}
    \label{fig:voronoi-right-full}
  \end{figure*}